%% file: main.tex
\documentclass[11pt,letterpaper]{article}

\usepackage[margin=1in]{geometry}        
\usepackage{setspace}                     
\usepackage{parskip}                      
\usepackage[T1]{fontenc}                  
\usepackage[utf8]{inputenc}               
\usepackage{lmodern} 
\usepackage{graphicx}
\usepackage{natbib}
\usepackage{authblk}
\usepackage{enumerate}
\usepackage{enumitem}
\usepackage{tikz}
\usetikzlibrary{decorations.pathreplacing}
\usetikzlibrary{math}

\usepackage{amsmath,amssymb,amsthm,mathtools,bbm}

\numberwithin{equation}{section}

\theoremstyle{plain}
\newtheorem{theorem}{Theorem}[section]
\newtheorem{lemma}[theorem]{Lemma}
\newtheorem{proposition}[theorem]{Proposition}

\theoremstyle{definition}
\newtheorem{definition}{Definition}

\theoremstyle{remark}

\makeatletter
\newtheorem*{rep@theorem}{\rep@title}
\newcommand{\newreptheorem}[2]{%
\newenvironment{rep#1}[1]{%
 \def\rep@title{#2 \ref{##1} (extended version)}%
 \begin{rep@theorem}}%
 {\end{rep@theorem}}}
\makeatother

\newreptheorem{theorem}{Theorem}

\usepackage[ruled, linesnumbered]{algorithm2e}
\usepackage{xcolor}
\definecolor{lblue}{RGB}{40,103,178}
\definecolor{cred}{RGB}{177,4,14}
\definecolor{sgreen}{RGB}{46,139,87}

\usepackage[
  colorlinks,
  linkcolor=blue,
  citecolor=blue,
  urlcolor=blue,
  breaklinks
]{hyperref}

\def\E{\mathbb{E}}

\def\R{\mathbb{R}}
\def\N{\mathbb{N}}

\newcommand{\bX}{\mathbf{X}}
\newcommand{\bY}{\mathbf{Y}}
\newcommand{\bZ}{\mathbf{Z}}

\newcommand{\bW}{\mathbf{W}}
\newcommand{\bxi}{\boldsymbol{\xi}}
\newcommand{\dfucb}{\textsc{DF-UCB}\,}
\newcommand{\dflcb}{\textsc{DF-LCB}\,}

\newcommand{\EE}[1]{\mathbb{E}\left[{#1}\right]}
\newcommand{\EEst}[2]{\mathbb{E}\left[{#1}\  \middle| \ {#2}\right]}
\newcommand{\Ep}[2]{\mathbb{E}_{{#1}}\left[{#2}\right]}

\newcommand{\PP}[1]{\mathbb{P}\left\{{#1}\right\}}
\newcommand{\PPst}[2]{\mathbb{P}\left\{{#1}\  \middle| \ {#2}\right\}}
\newcommand{\Ppst}[3]{\mathbb{P}_{{#1}}\left\{{#2}\  \middle| \ {#3}\right\}}
\newcommand{\Pp}[2]{\mathbb{P}_{{#1}}\left\{{#2}\right\}}

\newcommand{\One}[1]{{\mathbbm{1}}\left\{{#1}\right\}}
\newcommand{\one}[1]{{\mathbbm{1}}_{{#1}}}
\newcommand{\iidsim}{\stackrel{\textnormal{iid}}{\sim}}

\newcommand\independent{\protect\mathpalette{\protect\independenT}{\perp}}
\def\independenT#1#2{\mathrel{\rlap{$#1#2$}\mkern2mu{#1#2}}}

\newcommand{\dtv}{\mathrm{d}_{\mathrm{TV}}}
\newcommand{\dw}{\mathrm{d}_{\mathrm{W}}}
\newcommand{\Ph}{\widehat{P}}
\newcommand{\Qh}{\widehat{Q}}

\usepackage[nameinlink,noabbrev]{cleveref}
\crefname{assumption}{assumption}{assumptions}

\newcommand{\papertitle}{Distribution-free certificate of similarity between distributions\\ with blurred total variation distance}
\newcommand{\paperauthorA}{Rohan Hore}
\newcommand{\paperauthorB}{Rina Foygel Barber}
\newcommand{\affilOne}{Department of Statistics, Stanford University}
\newcommand{\affilTwo}{Department of Statistics, University of Chicago}

\newcommand{\corrEmail}{rhore@stanford.edu}
\newcommand{\keywordslist}{ Distribution testing, equivalence testing, distribution-free, total variation, TV distance, blurred TV distance.}

\title{\papertitle}

\author[1]{\paperauthorA\thanks{Corresponding author: \corrEmail}}
\author[2]{\paperauthorB}
\affil[1]{\affilOne}\affil[2]{\affilTwo}

\date{\today}

\newcommand{\paperabstract}[1]{%
  \begin{abstract}
    #1
  \end{abstract}
}
\newcommand{\paperkeywords}[1]{%
  \vspace{0.5em}
  \noindent\textbf{Keywords:} #1
}

\usepackage[nameinlink,noabbrev]{cleveref}

\begin{document}

\maketitle

\paperabstract{%
Classical two-sample tests aim to determine whether two distributions are equal or not based on samples from each one. Even in an distribution-free setting, well-known methods such as the permutation test can provide statistical evidence that the two distributions differ. However, failure to reject equality does not necessarily establish that they are similar. In this work, we study the complementary problem of certifying distributional similarity with finite-sample confidence, without imposing assumptions on the underlying distributions. We show, however, that for continuous distributions, any distribution-free upper confidence bound on the total variation (TV) distance must be trivial with high probability, even if the distributions are identical. Thus, certifying similarity in TV with a distribution-free guarantee is practically impossible. We therefore examine the blurred TV distance, a relaxation of TV distance that enables us to quantify similarity after applying a common blurring operation while retaining the interpretation of TV. We construct distribution-free upper and lower confidence bounds on blurred TV, establish their finite-sample validity, and examine its properties in high dimensions.}

\paperkeywords{\keywordslist}

\section{Introduction}\label{sec:intro}

Suppose we observe two independent samples $X_1,\dots,X_n\iidsim P$ and $Y_1,\dots,Y_m\iidsim Q$, where $P$ and $Q$ are distributions on $\R^d$. A fundamental problem in classical statistics and modern machine learning is the \emph{two-sample testing problem},
\[
H_0: P=Q \quad \text{vs.} \quad H_1: P\neq Q,
\]
which seeks to determine whether the two distributions are identical. This problem naturally arises in numerous real-world applications, spanning fields like genomics \citep{stegle2010robust}, finance \citep{horvath2013estimation}, and generative modeling \citep{bischoff2024practical}. Despite its long history, constructing a \emph{powerful} non-parametric test that is not overly conservative remains a non-trivial task. This is particularly difficult in high dimensions, but remains nontrivial even for small $d$ if we wish to avoid placing stringent assumptions on the distributions.

Despite these challenges, permutation-based two-sample tests offer a rigorous way to reject $H_0$ while controlling the Type~I error \citep{ramdas2017wasserstein,schrab2023mmd}. However, failure to reject does not establish that $P$ and $Q$ are equal, or even sufficiently similar. Positive evidence of similarity is important, for example, when validating model predictions \citep{robinson2004model} or assessing whether synthetic data reproduce the distribution of the original data \citep{snoke2018general}. In this paper, we focus on this complementary problem and seek to certify distributional similarity with finite-sample, distribution-free guarantees

A convenient way to reformulate the problem is by considering some measure of distance $D(P,Q)$, such as total variation (TV) distance, Kullback--Leibler (KL) divergence, or Wasserstein distance. For a \emph{proper} measure $D$, meaning $D(P,Q)=0$ if and only if $P=Q$, we seek to construct high-probability upper and lower bounds on $D(P,Q)$ based on the combined data $\mathcal{D}_{n,m} := (X_1,\ldots,X_n,Y_1,\ldots,Y_m)$. These bounds provide a quantitative assessment of distributional similarity, with a small upper bound certifying closeness under $D$. Ideally, we would wish to derive such bounds without imposing assumptions on the underlying distributions, so that they remain valid uniformly over all $P,Q\in\mathcal{P}_d$, where $\mathcal{P}_d$ denotes the set of Borel probability distributions supported on $\R^d$. This leads to the following notion of distribution-free confidence bounds.

\begin{definition}[Distribution-free confidence bounds]\label{def:DF_confidence_bounds}
For any $\alpha\in(0,1)$, a statistic $\hat{L}_\alpha$ is said to form a distribution-free lower confidence bound (\dflcb) on $D$ at level $1-\alpha$ if, for all $n,m\in\mathbb{N}$ and all $P,Q\in\mathcal{P}_d$,
\[\Pp{P^n\times Q^m\times\textnormal{Unif}[0,1]}{D(P,Q)\ge \hat{L}_\alpha(\mathcal{D}_{n,m},\zeta)} \ge 1-\alpha.\]
Similarly, $\hat{U}_\alpha$ is said to form a distribution-free upper confidence bound (\dfucb) on $D$ at level $1-\alpha$ if, for all $n,m\in\mathbb{N}$ and all $P,Q\in\mathcal{P}_d$, 
\[\Pp{P^n\times Q^m\times\textnormal{Unif}[0,1]}{D(P,Q)\le \hat{U}_\alpha(\mathcal{D}_{n,m},\zeta)} \ge 1-\alpha.\]
\end{definition}
In these definitions, the statistics $\hat{L}_\alpha$ and $\hat{U}_\alpha$ are each a function of the observed data $\mathcal{D}_{n,m}$, in addition to a random seed $\zeta\sim\textnormal{Unif}[0,1]$ that allows for randomization in the construction of the statistics, if desired.

Relating these confidence bounds to the original testing problem, note that a positive value of the \dflcb provides evidence against the null $H_0:P=Q$. In contrast, a small value of the \dfucb supports the claim that $D(P,Q)$ is small. Note that without additional assumptions, in finite samples, we cannot provide a nontrivial distribution-free certification of exact equality: distinct distributions can induce arbitrarily similar empirical laws for the observed samples. Thus, even when $D$ is proper, we cannot hope to statistically conclude the sharp statement $D(P,Q)=0$. Therefore, in this work, we relax our aim to certify approximate equality by providing an upper bound on $D(P,Q)$. More concretely, for any fixed tolerance $\epsilon>0$, we can conclude that $D(P,Q)\leq\epsilon$ with confidence $1-\alpha$ whenever $\hat{U}_\alpha(\mathcal{D}_{n,m},\zeta)\leq\epsilon$. This provides a statistically valid statement of similarity rather than merely a failure to detect a difference. This perspective is closely related to equivalence testing, in which one consider discrepancies exceeding a prescribed tolerance as the null hypothesis, so that rejection provides positive evidence of similarity \citep{munk1998nonparametric,bastian2024testing}.


\paragraph{Total variation distance.} The total variation (TV) distance, $\dtv(P,Q)=\sup_{A\subseteq\R^d}|P(A)-Q(A)|$, is a canonical choice for the measure of distance $D$. TV-based analyses are central in classical statistics because of their fundamental connection to hypothesis testing and robustness of inference procedures \citep{lehmann2005testing}. More recently, TV has re-emerged in modern machine learning, particularly in the analysis of diffusion-based generative models, where it is used as a strong metric to study distributional convergence and stability properties of the learned dynamics \citep{liang2025low,li2024d}.
However, TV distance is difficult to estimate from finite samples and is highly sensitive to fine-scale differences between $P$ and $Q$, making any inference on it statistically challenging in nonparametric regimes. We next elaborate on this difficulty and show that constructing a meaningful \dfucb for $\dtv$ is, in general, theoretically infeasible. 

\subsection{Non-trivial \dfucb on total variation distance is impossible}\label{sec:DF_UCB_original_TV}
Although the definitions of a \dfucb and \dflcb appear symmetric (see Definition~\ref{def:DF_confidence_bounds}), the two problems are fundamentally different. In particular, for TV distance, we will now see that a non-trivial \dfucb is often impossible without additional assumptions. Indeed, when $P$ and $Q$ are continuous distributions, finite samples of sizes $n$ and $m$ are almost surely distinct, and thus one cannot rule out the possibility that $P$ and $Q$ are mutually singular (i.e., the two distributions have non-overlapping supports, and thus $\dtv(P,Q)=1$). As a consequence, in the absence of structural assumptions on the distributions, any distribution-free upper confidence bound on $\dtv(P,Q)$ must be trivial, in the following sense:
\begin{theorem}\label{thm:hardness}
Fix $\alpha\in[0,1]$, any $d\geq 1$, and any $n,m\geq 1$. Let $\hat{U}_\alpha\in [0,1]$ be any (possibly randomized) distribution-free upper confidence bound for $\dtv(\cdot,\cdot)$. Then, for any pair of distributions $P,Q\in \mathcal{P}_d$ satisfying $\textnormal{atom}(P)\cap\textnormal{atom}(Q)=\varnothing$,
\[
\Pp{P^n\times Q^m\times\textnormal{Unif}[0,1]}
{\hat{U}_\alpha(\mathcal{D}_{n,m},\zeta)=1}\;\ge\; 1-\alpha .
\]
\end{theorem}
\noindent
(Here, $\textnormal{atom}(P)=\{x\in\mathbb{R}^d: P(\{x\})>0\}$ denotes the set of atoms of $P$.) In other words, any valid \dfucb is inevitably as large as the following completely trivial solution: simply return $\hat{U}_\alpha=1$ with probability $1-\alpha$, and $\hat{U}_\alpha=0$ otherwise, independently of the data.

Although here we state this negative result in the specific context of constructing a \dfucb, the statistical difficulty of working with total variation distance has been well documented in the literature through a range of different frameworks. In particular, when $P$ and $Q$ are supported on a finite alphabet of size $K$, consistent estimation of the TV distance requires sample sizes satisfying $n,m \gg K/\log(K)$, and even distinguishing between $\dtv(P,Q)=0$ and $\dtv(P,Q)\geq c$ requires $n,m\gtrsim K^{2/3}$ \citep{jiao2018minimax,valiant2017estimating,chan2014optimal}. Consequently, even in discrete settings, a large support size means that constructing tight, finite-sample confidence bounds for TV distance remains challenging without strong assumptions; this challenge is even more severe for distributions that are continuous. In particular, we emphasize that this challenge is not simply due to high dimensionality: the result of Theorem~\ref{thm:hardness} holds even in dimension $d=1$, for arbitrarily large sample sizes $n,m$.

In contrast, constructing a \dflcb for $\dtv$ is often comparatively straightforward. For any measurable set $A \subseteq \mathbb{R}^d$, it holds that $\dtv(P,Q) \ge |P(A) - Q(A)|$. Given samples $X_1,\ldots,X_n \iidsim P$ and $Y_1,\ldots,Y_m \iidsim Q$, the empirical probabilities consistently estimate $P(A)$ and $Q(A)$, respectively. For any fixed $A$, we can therefore derive tight, distribution-free lower bounds on $|P(A)-Q(A)|$---and hence on $\dtv(P,Q)$. Unlike in the case of upper confidence bounds, such lower bounds can be non-trivial for suitably chosen sets $A$ and distributions $P,Q$.

This asymmetry between upper and lower confidence bounds motivates the central idea of this work: to introduce a principled relaxation of TV distance that permits non-trivial, assumption-free guarantees for providing an upper as well as a lower bound, while retaining the interpretability of the classical TV metric.

\subsection{The blurred total variation distance}\label{sec:blurred_TV_defn}

The statistical difficulty of working with the classical TV distance has motivated a substantial body of work on alternative discrepancies that trade metric strength for improved estimability. Prominent examples include kernel-based distances such as the maximum mean discrepancy (MMD) \citep{gretton2012kernel,schrab2023mmd}, as well as optimal transport--based distances such as the Wasserstein distance \citep{ramdas2017wasserstein,wang2021two}. 
While these metrics admit distribution-free estimation, often in a computationally tractable manner, their interpretation is fundamentally different from the natural “worst-case discrepancy in event probabilities” characterization of the TV measure.

In this paper, we investigate a different relaxation, the \emph{blurred total variation} (blurred TV) distance, which compares smoothed versions of the underlying distributions and can be viewed as an approximation to the TV distance while retaining its worst-case notion of discrepancy between distributions.
We begin by defining the class of kernels used to induce smoothing. Let $\mathcal{K}_d$ denote the collection of all probability densities on $\R^d$, which we refer to as \emph{kernels}. Given a kernel $\psi\in\mathcal{K}_d$ and a bandwidth parameter $h>0$, define the scaled kernel $\psi_h(u) = h^{-d}\psi(u/h)$. A canonical and widely used choice of kernel is the Gaussian density:
\begin{equation}\label{kernel:gaussian}
\psi(u) = \frac{1}{(2\pi)^{d/2}} \exp\bigl(-\|u\|^2_2/2\bigr).
\end{equation}
\begin{definition}[Blurred total variation distance]\label{def:blurred_tv}
Let $\psi\in\mathcal{K}_d$ and $h>0$. The \emph{blurred total variation distance} between $P,Q\in\mathcal{P}_d$ is defined as
\begin{equation}\label{eqn:smoothed_TV}
\dtv^h(P,Q) := \dtv(P\ast\psi_h,\, Q\ast\psi_h).
\end{equation}
\end{definition}

Here $\ast$ denotes the convolution operation: we can interpret $P\ast\psi_h$ as the distribution of $X+h\xi$, where $X\sim P$ and, independently, $\xi\sim\psi$. We will also write $\dtv^0(P,Q)$ to denote $\dtv(P,Q)$. Thus, blurred TV works with the same total variation measure, but for random variables measured at resolution~$h$. More concretely, given $X\sim P$ and $Y\sim Q$, and given noise variables $\xi,\xi'\iidsim\psi$ (drawn independently of $X,Y$), the blurred TV distance $\dtv^h(P,Q)$ is equivalent to measuring \emph{the TV distance of the blurred data}, given by $X+\xi$ and $Y+\xi'$.

In contrast to the metrics mentioned at the start of this section, which also permit distribution-free inference, blurred TV modifies TV itself rather than replacing it with an entirely different notion of distance. It retains TV’s probabilistic interpretation: it measures the worst-case discrepancy in event probabilities after applying the same mild blurring operation to both distributions.

Blurred TV is always a pseudometric, meaning that it is nonnegative and symmetric, satisfies the triangle inequality, and equals zero for identical distributions. Under a mild condition on the kernel $\psi$, it is also a proper metric on $\mathcal{P}_d$, i.e., it equals zero if and only if the underlying distributions are equal. We formally state this result below, with the proof deferred to Appendix~\ref{app:proof_of_blurred_TV_metric}.

\begin{proposition}\label{prop:blurred_TV_metric}
For any kernel $\psi$ and bandwidth $h>0$, $\dtv^h$ is a pseudometric on $\mathcal{P}_d$. Additionally, if the characteristic function of $\psi$ is nonzero on a dense subset of $\R^d$, then $\dtv^h$ is a metric, i.e.,
\[\dtv^h(P,Q)=0 \quad\Longleftrightarrow\quad P=Q.\]
\end{proposition}
This condition on $\psi$ holds for a broad class of kernels. For example, the characteristic functions of Gaussian and Laplace kernels are nonzero everywhere, while that of the box kernel is nonzero on a dense subset of $\R^d$.

\paragraph{Prior work on blurred TV.} The blurred TV distance, and related definitions, have been studied previously in several contexts. In information theory, total variation after smoothing arises naturally in the study of contraction properties of additive-noise channels \citep{du2017strong}. \citet{lugosi2022generalization} uses blurred TV to obtain generalization bounds on learning algorithms. More generally, \citet{kanamori2024robust} introduces smoothed total variation as a class of integral probability metrics. The work of \citet{polyanskiy2015dissipation} establishes a central limit theorem in blurred TV, while \citet{goldfeld2020convergence,goldfeld2020limit} study convergence of the empirical measure in the blurred TV metric and its limiting distribution, particularly for the Gaussian kernel~\eqref{kernel:gaussian}. 

\subsection{Our contributions}
In this work, we position blurred TV as a principled relaxation of the classical TV distance for distribution-free inference, and demonstrate that meaningful and practically usable finite-sample confidence bounds for blurred TV are possible, even though classical TV suffers from a fundamental hardness (Theorem~\ref{thm:hardness}). Our main contributions are:
\begin{itemize}
    \item We construct distribution-free lower and upper confidence bounds (\dflcb\ and \dfucb) for the blurred TV distance. These bounds are designed to satisfy \textbf{assumption-free, finite-sample validity}: the proposed confidence bounds are fully distribution-free and valid uniformly over all distributions $P,Q\in\mathcal{P}_d$, without requiring smoothness, moment, or other structural assumptions, with guarantees that hold for arbitrary sample sizes $n,m$ and dimension $d$.
    \item In the course of developing these confidence bounds, we establish several structural and analytical properties of the blurred TV distance that further justify its role as a principled relaxation of classical TV. 
    \item We study the behavior of blurred TV in high dimensions and show that, although distribution-free inference can lose power when the dimension $d$ grows too quickly relative to the sample sizes $n$ and $m$, the blurred TV distance depends primarily on the intrinsic dimension of the data. This suggests that meaningful distribution-free inference remains possible even in high dimensions when the distributions are supported on a common low-dimensional subspace.
\end{itemize}

\subsection{Organization of the paper}
The remainder of the paper is organized as follows. First, in Section~\ref{sec:prop_blurred_TV}, we establish key properties of both the oracle and empirical blurred TV distances. In Section~\ref{sec:confidence_bounds}, we present distribution-free confidence bounds on blurred TV. 
Finally, in Section~\ref{sec:highD_blurred_TV}, we study the behavior of these bounds in high dimensions and conclude with a short discussion in Section~\ref{sec:discussion}. All proofs are deferred to the appendix.

\section{Properties of blurred TV distance}\label{sec:prop_blurred_TV}
In this section, we establish several fundamental properties of the blurred TV distance that motivate its choice as a principled relaxation of the classical TV distance. Together, these features position blurred TV as a natural and tractable surrogate for TV in nonparametric settings.

\subsection{The role of the bandwidth}\label{sec:prop_oracle_blurred_TV}
We begin by establishing several properties of the blurred TV distance that clarify its relationship to the classical total variation metric, and in particular, the role of the bandwidth parameter $h$ in governing this transition.
\begin{proposition}\label{prop:blurred_tv&tv_properties}
Fix distributions $P,Q\in \mathcal{P}_d$ and a kernel $\psi\in \mathcal{K}_d$. Then:
\begin{enumerate}[label=(\roman*)]
    \item $\dtv^h(P,Q)\le \dtv(P,Q)$ for all $h\geq 0$;
    \item $\lim_{h\to 0}\dtv^h(P,Q) = \dtv(P,Q)$;
    \item $\lim_{h\to\infty} \dtv^h(P,Q) = 0$;
    \item the map $h\mapsto \dtv^h(P,Q)$ is continuous on $h\in[0,\infty)$.
\end{enumerate}
\end{proposition}
Part (i) is an instance of the data-processing inequality for TV distance: adding independent noise through the same kernel can only decrease the separation between two distributions.
Parts (ii) and (iii) illustrate a fundamental trade-off in the choice of $h$: although a larger bandwidth $h$ facilitates statistically tractable, distribution-free inference due to greater smoothing, a smaller choice of bandwidth $h$ retains finer resolution with respect to the original distributions. We will return to this point in our results and experiments below.

Based on Proposition~\ref{prop:blurred_tv&tv_properties}, one might ask whether $\dtv^h(P,Q)$ decreases monotonically as the bandwidth $h$ increases, reflecting the increasing degree of smoothing induced by convolution. In particular, the commonly used Gaussian kernel satisfies this type of monotonicity property:
\begin{proposition}\label{prop:gaussian_monotonicity}
Suppose $\psi$ is the Gaussian kernel defined in \eqref{kernel:gaussian}. Then, for any two distributions $P, Q\in \mathcal{P}_d$, the blurred TV distance $\dtv^h(P,Q)$ is monotonically decreasing in the bandwidth $h$.
\end{proposition}

This monotonicity for the Gaussian kernel follows from the infinite divisibility of Gaussian distributions: Gaussian noise with bandwidth $h$ can be represented as the sum of independent Gaussian noises with bandwidths $0<h_1,h_2<h$ satisfying $h_1^2+h_2^2=h^2$.
However, for other choices of $\psi$, the blurred TV distance is not necessarily monotone in $h$ in general. Figure~\ref{fig:monotonicity_blurred_tv} illustrates this phenomenon with an example.

\begin{figure}[t]
    \centering
    \includegraphics[scale=0.75]{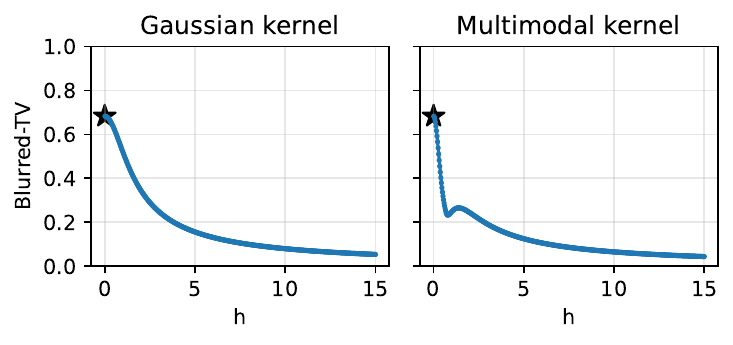}
    \caption{Near-monotonic behavior of blurred TV with bandwidth $h$. Here $P=\mathcal{N}(1,1)$, $Q=\mathcal{N}(-1,1)$, with $\dtv(P,Q) = \Phi(1)-\Phi(-1) \approx 0.683$, marked with a $\star$ in the figure. $\psi$ is either the Gaussian kernel (left), or a multimodal kernel, given by a density of the mixture distribution $\tfrac13\,\mathcal{N}(-4,1)+\tfrac13\,\mathcal{N}(0,1)+\tfrac13\,\mathcal{N}(4,1)$ (right).
    } 
    \label{fig:monotonicity_blurred_tv}
\end{figure}

Proposition~\ref{prop:blurred_tv&tv_properties}~(ii) suggests that, for small $h$, blurred TV remains close to the original TV distance. In Appendix~\ref{app:blurred_tv_approximation}, we provide an explicit bound on $\dtv(P,Q)-\dtv^h(P,Q)$ under a mild smoothness condition on the underlying distributions; under suitable conditions, this gap is of order $O(h)$.

\subsection{Convergence of the empirical blurred TV}\label{sec:in_blurred_TV_convergence}
In practice, we do not have access to the true distributions $P$ and $Q$, and instead can only work with the empirical distributions $\Ph_n=\frac{1}{n}\sum_{i=1}^n \delta_{X_i}$ and $\Qh_m=\frac{1}{m}\sum_{j=1}^m\delta_{Y_j}$ of the available data $X_1,\dots,X_n$ and of $Y_1,\dots,Y_m$, respectively. The appeal of blurred TV is that (unlike TV) it permits approximation via the empirical distributions, without assumptions on the original distributions. Here we develop some of these properties to build towards our main results, constructing the \dflcb and \dfucb later on.

First, we verify that the empirical measure $\Ph_n$ converges to $P$ in blurred TV, as $n\to\infty$.
\begin{theorem}\label{thm:blurred_TV_convergence_asymptotic}
     For any distribution $P\in \mathcal{P}_d$, any kernel $\psi\in\mathcal{K}_d$, and any fixed $h>0$,
     \[\lim_{n\to\infty} \EE{\dtv^h(\Ph_n,P)} = 0.\]
 \end{theorem}
This result is closely connected to classical $L_1$ convergence results for kernel density estimators \citep{devroye1983equivalence}.
It is also possible to give more precise, quantitative results, establishing a specific rate of convergence under mild assumptions on $P,Q$ and on $\psi$. For instance, \citet[Proposition~2]{goldfeld2020convergence} establishes this type of result for Gaussian kernels; see also Appendix~\ref{app:additional_results} for a quantitative result for general $\psi$.

The above theorem implies that, with respect to the blurred TV metric, the empirical distributions $\Ph_n$ and $\Qh_m$ offer accurate approximations to $P$ and $Q$. In particular, this implies that $\dtv^h(P,Q)$ should be possible to estimate via the corresponding empirical quantity, $\dtv^h(\Ph_n,\Qh_m)$. Our next result makes this intuition precise by establishing bounds on the expected value of this empirical estimator. Before proceeding, we first define the split empirical measures
\[
\Ph_{n/2}^{(1)} := \frac{1}{\lfloor n/2\rfloor}\sum_{i=1}^{\lfloor n/2\rfloor}\delta_{X_i}, \qquad \Ph_{n/2}^{(2)} := \frac{1}{\lceil n/2\rceil}\sum_{i=\lfloor n/2\rfloor+1}^{n}\delta_{X_i},
\]
based on two (approximately) equal halves of the sample $X_1,\dots,X_n$ of data from $P$ (where we implicitly assume $n\geq 2$ in order for this split to be defined). The analogous quantities $\Qh_{m/2}^{(1)}$ and $\Qh_{m/2}^{(2)}$ are defined similarly using the data $Y_1,\dots,Y_m$ (when $m\geq 2$).

\begin{proposition}\label{prop:E_inequality}
Fix any $d\geq1$, any  distributions $P,Q\in \mathcal{P}_d$, any kernel $\psi\in\mathcal{K}_d$, and any bandwidth $h>0$. Then, for all $n,m\geq 1$,
\[
\EE{\dtv^h(\Ph_n,\Qh_m)} \ge \dtv^h(P,Q) \ge \EE{\dtv^h(\Ph_n,\Qh_m)} - \Delta_{n,m,h},
\]
where
\begin{equation}\label{eq:delta_nmh}
\Delta_{n,m,h}:=\EE{\dtv^h(\Ph^{(1)}_{n/2},\Ph^{(2)}_{n/2})}+\EE{\dtv^h(\Qh^{(1)}_{m/2},\Qh^{(2)}_{m/2})},
\end{equation}
or we simply take $\Delta_{n,m,h}=1$ if $\min\{n,m\}=1$.
\end{proposition}

The gap $\Delta_{n,m,h}$ quantifies the intrinsic variability within the $X$ and $Y$ samples and admits a natural, data-driven estimate via sample splitting. By Theorem~\ref{thm:blurred_TV_convergence_asymptotic}, $\Delta_{n,m,h}$ converges to zero as $n,m\to\infty$ (when $h>0$ is held constant). Consequently, this suggests that the empirical blurred TV distance $\dtv^h(\Ph_n,\Qh_m)$ provides an accurate approximation of the oracle blurred TV measure $\dtv^h(P,Q)$---but thus far, we have only examined its expected value.
In the next section, we show that the empirical blurred TV distance can be used to construct both distribution-free lower and upper confidence bounds for $\dtv^h(P,Q)$.

\section{\dfucb\ and \dflcb\ for blurred TV}\label{sec:confidence_bounds}
In this section, we construct distribution-free confidence bounds for the blurred TV distance.

\begin{theorem}\label{thm:naive_confidence_bounds}
Fix any $d\geq 1$, $n,m\geq 1$, $\psi\in\mathcal{K}_d$, $h>0$, and $\alpha\in(0,1)$. Define
\[\hat{U}_\alpha(\mathcal{D}_{n,m}):= \dtv^h(\Ph_n,\Qh_m)
+ \epsilon_{n,m,\alpha}\]
and\[
\hat{L}_\alpha(\mathcal{D}_{n,m}):= \max\Bigl\{\dtv^h(\Ph_n,\Qh_m)
- \dtv^h(\Ph^{(1)}_{n/2},\Ph^{(2)}_{n/2})
- \dtv^h(\Qh^{(1)}_{m/2},\Qh^{(2)}_{m/2})
- 3\,\epsilon_{n-1,m-1,\alpha},\,0\Bigr\},\]
where $\epsilon_{n,m,\alpha}:= \left(\frac{\log(1/\alpha)}{2}\left(\frac{1}{n}+\frac{1}{m}\right)\right)^{1/2}$.
Then $\hat{U}_\alpha(\mathcal{D}_{n,m})$ is a \dfucb for $\dtv^h(\cdot,\cdot)$ at level $1-\alpha$, and\footnote{The \dflcb can be interpreted as $0$ if $n=1$ or $m=1$ (i.e., if the split empirical measures are not defined).} $\hat{L}_\alpha(\mathcal{D}_{n,m})$ is a \dflcb for $\dtv^h(\cdot,\cdot)$ at  level $1-\alpha$.
\end{theorem}
The validity of these bounds follows from the fact that the blurred TV can be bounded via the expected value of empirical blurred TV, as in Proposition~\ref{prop:E_inequality}; concentration is then established via McDiarmid's inequality.

Although this result offers clean and useful bounds on the blurred TV, several aspects can be improved for practical use. We summarize the main issues below and indicate how they are addressed.
\begin{itemize}
    \item The empirical blurred TV $\dtv^h(\Ph_n,\Qh_m)$ is hard to compute exactly---numerical integration is possible for smooth kernels, but is computationally expensive. In Section~\ref{sec:monte_carlo_approximation}, we introduce a Monte Carlo approximation that is easy to compute and preserves distribution-free validity.
    \item The above bounds are valid for a fixed bandwidth $h>0$. In practice, it is often desirable to select $h$ in a data-dependent way. This motivates the refined confidence bounds in Section~\ref{sec:uniform_confidence_bounds} that are uniformly valid over a range of bandwidths, which we obtain via a monotonization approach.
    \item For any $P,Q$, $\dtv^h(P,Q)\to 0$ as $h$ increases. However, for fixed $n,m$, the error terms in the bounds above, $\epsilon_{n,m,\alpha}$, do not decay with $h$. To improve interpretability, we construct bounds in Section~\ref{sec:confidence_bounds_bandwidth_adaptive} that scale with $h$, to allow for a \dfucb  that is less conservative as $h\to\infty$.
\end{itemize}

Each of the three following subsections addresses the above three aims individually. However, we emphasize that all of these issues can be handled simultaneously; we defer the full construction to Appendix~\ref{app:additional_results} and here present only the individual results, for clarity.
Since the \dfucb and \dflcb are handled in a symmetric manner, for clarity we present only the \dfucb results in the remainder of this section, and give the corresponding \dflcb results in the appendix.

\subsection{Monte Carlo approximation of empirical blurred TV}\label{sec:monte_carlo_approximation}
We introduce a Monte Carlo approximation of the empirical blurred TV that enables practical implementation of the confidence bounds in Theorem~\ref{thm:naive_confidence_bounds} while preserving distribution-free validity.

First, observe that the empirical blurred TV admits the representation
\[
\dtv^h(\Ph_n,\Qh_m)
=\frac{1}{2}\Ep{\substack{W \sim \frac{1}{2}\Ph_n + \frac{1}{2}\Qh_m \\ \xi \sim \psi}}{\left|
\frac{\mathsf{d}(\Ph_n \ast \psi_h) - \mathsf{d}(\Qh_m \ast \psi_h)}
{\mathsf{d}\bigl(\frac{1}{2}\Ph_n \ast \psi_h + \frac{1}{2}\Qh_m \ast \psi_h\bigr)}
\bigl(W + h\xi\bigr)\right|}.
\]
Motivated by this identity, we define the Monte Carlo estimator
\[
\widehat{\dtv^h}(\Ph_n,\Qh_m;B)
:=\frac{1}{2B}\sum_{k=1}^B\left|
\frac{\mathsf{d}(\Ph_n \ast \psi_h) - \mathsf{d}(\Qh_m \ast \psi_h)}
{\mathsf{d}\bigl(\frac{1}{2}\Ph_n \ast \psi_h + \frac{1}{2}\Qh_m \ast \psi_h\bigr)}
\bigl(W_k + h\xi_k\bigr)\right|,
\]
where $W_1,\ldots,W_B \iidsim \frac{1}{2}\Ph_n + \frac{1}{2}\Qh_m$ and $\xi_1,\ldots,\xi_B \iidsim \psi$. Here $B$ denotes the Monte Carlo sample size. 
Note that both $\Ph_n\ast\psi_h$ and $\Qh_m\ast\psi_h$ are absolutely continuous with respect to the measure $\frac{1}{2}(\Ph_n\ast\psi_h+\Qh_m\ast\psi_h)$. We therefore write $\mathsf{d}(\Ph_n\ast\psi_h)/\mathsf{d}(\frac{1}{2}\Ph_n\ast\psi_h+\frac{1}{2}\Qh_m\ast\psi_h)$ for the corresponding Radon--Nikodym derivative, and we define $\mathsf{d}(\Qh_m\ast\psi_h)/\mathsf{d}\left(\frac{1}{2}\Ph_n\ast\psi_h+\frac{1}{2}\Qh_m\ast\psi_h\right)$  analogously. The \dfucb in Theorem~\ref{thm:naive_confidence_bounds} can be updated with only an additional $O(1/\sqrt{B})$ term.

\begin{theorem}\label{thm:monte_carlo_ucb}
Fix any $d\geq 1$, $n,m\geq 1$, $\psi\in\mathcal{K}_d$, $h>0$, and $\alpha\in(0,1)$. Define
\[
\hat{U}_\alpha(\mathcal{D}_{n,m},B):= \widehat{\dtv^h}(\Ph_n,\Qh_m;B)+ \epsilon_{n,m,\alpha/2}+ \epsilon_{B,\alpha/2},
\]
where $\epsilon_{B,\alpha}:=\left(\frac{\log(1/\alpha)}{2B}\right)^{1/2}$. Then $\hat{U}_\alpha(\mathcal{D}_{n,m},B)$ is a \dfucb for $\dtv^h(\cdot,\cdot)$ at confidence level $1-\alpha$.
\end{theorem}

\subsection{Uniform validity of confidence bounds across bandwidths}\label{sec:uniform_confidence_bounds}
We now construct a \dfucb and a \dflcb for $\dtv^h(P,Q)$ that are valid simultaneously for all bandwidths $h>0$. As a warm-up, consider a finite collection
\[
\mathcal{H}_L=\{h_0,h_1,\ldots,h_L\}, \qquad 0< h_0\le h_1\le\cdots\le h_L .
\]
Uniform validity over $\mathcal{H}_L$ follows by a union bound applied to Theorem~\ref{thm:naive_confidence_bounds}. In particular, $\hat{L}_{\alpha/(L+1)}(\mathcal{D}_{n,m})$ and $\hat{U}_{\alpha/(L+1)}(\mathcal{D}_{n,m})$ provide a \dflcb and a \dfucb for $\dtv^h(P,Q)$, uniformly for all $h\in\mathcal{H}_L$.

While simple, this approach becomes overly conservative when $L$ is large. We therefore propose an alternative construction leveraging monotonicity. Define the \emph{monotonized blurred TV} distances
\begin{equation}\label{eqn:monotonized_blurred_TV}
\dtv^{h,\mathrm{up}}(P,Q):= \sup_{h_1\in[h,\infty)} \dtv^{h_1}(P,Q),
\qquad\dtv^{h,\mathrm{lo}}(P,Q):= \inf_{h_1\in[0,h]} \dtv^{h_1}(P,Q).
\end{equation}
These quantities respectively provide upper and lower envelopes of $\dtv^h(P,Q)$. In particular,
\[
\dtv^{h,\mathrm{lo}}(P,Q)\;\le\;\dtv^h(P,Q)\;\le\;
\dtv^{h,\mathrm{up}}(P,Q)\;\le\;\dtv(P,Q),\qquad \text{for all } h>0,
\]
where the last inequality follows from Proposition~\ref{prop:blurred_tv&tv_properties}. Thus, $\dtv^{h,\mathrm{up}}$ positions itself as a uniform upper envelope that remains closer to the classical TV distance. Moreover, as illustrated in Figure~\ref{fig:monotonicity_blurred_tv}, $\dtv^h(P,Q)$ is typically monotone or nearly monotone in $h$, so the envelopes closely track the original blurred TV. This observation enables distribution-free confidence bounds that are uniformly valid over $h\in(0,\infty)$.

\begin{theorem}\label{thm:ucb_uniform}
Fix any $d\geq 1$, $n,m\geq 1$, $\psi\in\mathcal{K}_d$, and $\alpha\in(0,1)$. For all $h>0$, define \[\hat{U}^{h,\mathrm{up}}_\alpha(\mathcal{D}_{n,m}):= \dtv^{h,\mathrm{up}}(\Ph_n,\Qh_m)+ \epsilon_{n,m,\alpha/(n\wedge m)} + \frac{1}{n\wedge m}.\]
Then $\hat{U}^{h,\mathrm{up}}_\alpha(\mathcal{D}_{n,m})$ is \dfucb for $\dtv^h(\cdot,\cdot)$, uniformly for all $h\in(0,\infty)$, at confidence level $1-\alpha$ (that is, $\PP{\dtv^h(P,Q)\le \hat{U}^{h,\mathrm{up}}_\alpha(\mathcal{D}_{n,m})\ \forall \ h\in(0,\infty)} \ge 1-\alpha$).
\end{theorem}

\subsection{Bandwidth-adaptive confidence bounds at a fixed $h$}\label{sec:confidence_bounds_bandwidth_adaptive}
The confidence bounds in Theorem~\ref{thm:naive_confidence_bounds} are based on a bounded-differences argument that does not adapt to the variability of the empirical blurred TV. We now introduce variance-adaptive confidence bounds at a fixed bandwidth $h>0$. The key ingredient is a data-dependent variance proxy that captures the sensitivity of the kernel to translations.

For a kernel $\psi$, define its \emph{shift modulus} by $\omega_\psi(v):=\dtv(\delta_v\ast\psi,\psi)$ for $v\in \R^d$.
For commonly used kernels, including the Gaussian kernel, $\omega_\psi(v)$ grows at most linearly for small $\|v\|_2$, and is uniformly bounded, $\omega_\psi(v)\leq 1$. Using this quantity, we define
\begin{equation}\label{eqn:hat_sigma}
\hat\Sigma_{n,m}^{h}:= \frac{1}{n\binom{n}{2}}\sum_{1\le i<j\le n}\omega_\psi\!\left(\frac{X_i-X_j}{h}\right)^2
+\frac{1}{m\binom{m}{2}}\sum_{1\le i<j\le m}\omega_\psi\!\left(\frac{Y_i-Y_j}{h}\right)^2 .
\end{equation}
$\hat\Sigma_{n,m}^{h}$ captures the inherent variability of empirical blurred TV terms. By construction, $\hat\Sigma_{n,m}^{h}\leq \frac{1}{n}+\frac{1}{m}$ for any $h>0$, but $\hat\Sigma_{n,m}^{h}$ may be substantially smaller if $h$ is large.
\begin{theorem}\label{thm:ucb_bandwidth_adaptive}
Fix any $d\geq 1$, $n,m\geq 1$, $\psi\in\mathcal{K}_d$, $h>0$, and $\alpha\in(0,1)$. Define 
\[
\hat{U}_{\alpha}(\mathcal{D}_{n,m},\hat\Sigma_{n,m}^{h}):=\dtv^h(\Ph_n,\Qh_m)+ C_\alpha\left\{\bigl(\hat\Sigma_{n,m}^{h}\bigr)^{1/2}  + \frac{1}{n\wedge m}\right\},
\]
where $C_\alpha$ depends only on $\alpha$ and is defined in the proof.
Then $\hat{U}_{\alpha}(\mathcal{D}_{n,m},\hat\Sigma_{n,m}^{h})$
is a \dfucb\ for $\dtv^h(\cdot,\cdot)$ at confidence level $1-\alpha$.
\end{theorem}
We see that this improves on the original \dfucb of Theorem~\ref{thm:naive_confidence_bounds}, in terms of the magnitude of the terms added to the estimate $\dtv^h(\Ph_n,\Qh_m)$: this new result adds a margin of error that is $\mathcal{O}(\bigl(\hat\Sigma_{n,m}^{h}\bigr)^{1/2} + \frac{1}{n\wedge m})$, rather than $\mathcal{O}(\frac{1}{\sqrt{n\wedge m}})$ as in Theorem~\ref{thm:naive_confidence_bounds}, which is an improvement as $h\to\infty$.

\subsection{Numerical experiment}\label{sec:simulation}
We illustrate the practical performance\footnote{The code to reproduce all the experiments is available at \href{https://colab.research.google.com/drive/1rWJD1NashERaQFkB3yQpdkPep5e2XLed?usp=sharing}{this Colab notebook}.} of our confidence bounds in a simple Gaussian setting, with $P=\mathcal{N}(\beta,1)$ and $Q=\mathcal{N}(-\beta,1)$. Figure~\ref{fig:confidence_bounds} displays the Monte Carlo confidence bounds from Theorem~\ref{thm:monte_carlo_ucb} (see Appendix~\ref{app:proof_of_monte_carlo_confidence_bounds} for the full statement that includes both the \dfucb and the \dflcb), together with the empirical and oracle blurred TV, for different values of $\beta$.

\begin{figure}[t]
    \centering
    \includegraphics[width=0.87\linewidth]{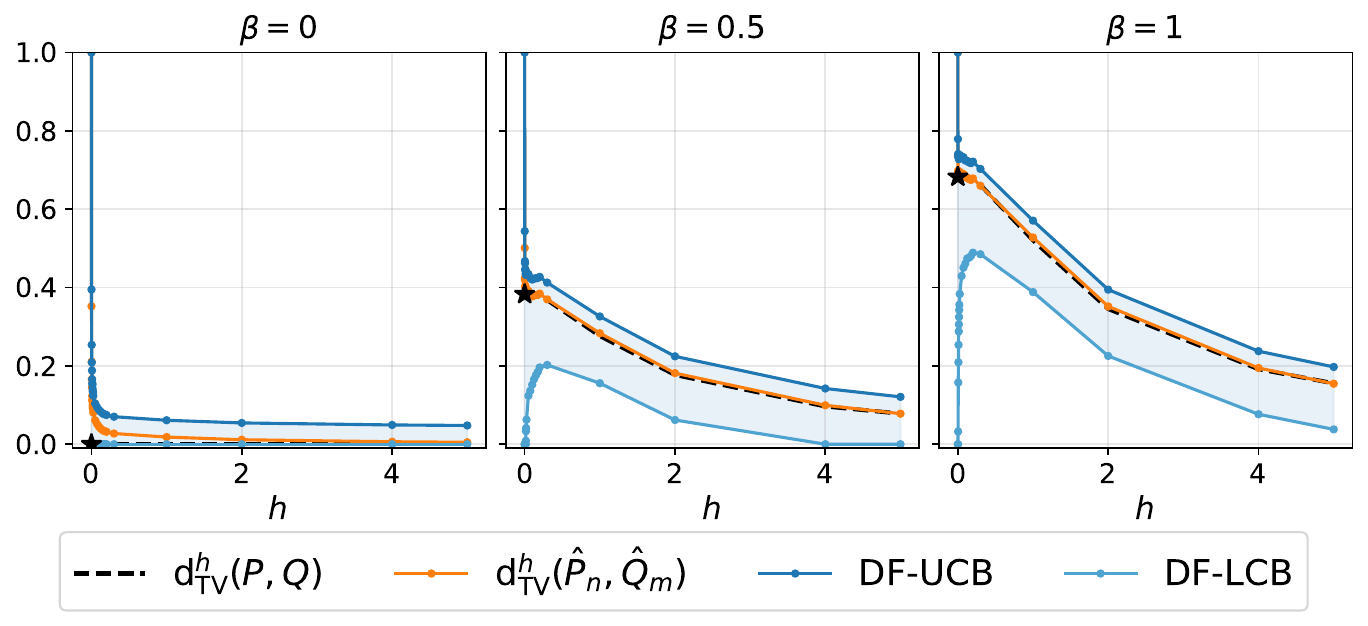}
    \caption{Monte Carlo based confidence bounds on $\dtv^h(P,Q)$. In each plot, $\dtv(P,Q)$ is marked by a $\star$ symbol. See Section~\ref{sec:simulation} for simulation details.}
    \label{fig:confidence_bounds}
\end{figure}

The empirical blurred TV $\dtv^h(\Ph_n,\Qh_m)$ closely tracks the oracle quantity $\dtv^h(P,Q)$, confirming its accuracy as an estimator. Although classical TV admits no non-trivial distribution-free upper bound (Theorem~\ref{thm:hardness}), the \dfucb for blurred TV remains informative, staying well below $1$ even at small bandwidths $h>0$.

\subsection{How to choose a bandwidth?}\label{sec:bandwidth_choice}

In our framework, the bandwidth $h$ determines the resolution at which distributional similarity is assessed. Proposition~\ref{prop:blurred_tv&tv_properties}~(ii) shows that $\dtv^h(P,Q)\to\dtv(P,Q)$ as $h\to0$, suggesting that a small bandwidth keeps the blurred TV target close to classical TV. However, with finite samples, taking $h$ too small can make the confidence bounds uninformative. In particular, the leftmost panel of Figure~\ref{fig:confidence_bounds} shows that empirical blurred TV can be close to $1$ for small $h$, even when $P=Q$.

If we use the bounds from Theorem~\ref{thm:ucb_uniform}, which are uniformly valid across bandwidths, we may select $h$ in a data-adaptive manner. For example, we may pool the $\bX$ and $\bY$ samples together and randomly split them into two parts, so that both splits now represent the same population. We may then choose the smallest candidate bandwidth at which the upper confidence bound computed from these reshuffled samples falls below a prescribed tolerance $\epsilon$. On the other hand, in order to use the tighter fixed-bandwidth bounds from Theorem~\ref{thm:monte_carlo_ucb} or~\ref{thm:ucb_bandwidth_adaptive}, the bandwidth $h$ cannot depend on the data---but we may choose to instead use sample splitting, using half the data to select a useful bandwidth $h$ (i.e., a value of $h$ for which the blurred TV bound appears likely to be informative), and the remaining data to perform inference on $\dtv^h(P,Q)$.

\section{Blurred TV in high dimensions}\label{sec:highD_blurred_TV}
The distribution-free lower and upper confidence bounds for blurred TV established in the previous section are valid for any fixed dimension $d$. However, as $d$ increases, samples from $P$ and $Q$ become increasingly diffuse, and consequently the empirical blurred TV tends to approach $1$ even for moderate choices of the bandwidth, turning the resulting \dfucb essentially uninformative. This raises a natural question: is this behavior an artifact of the empirical estimator we take, or does it reflect an intrinsic limitation of inference on the blurred TV in high dimensions? 

In this section, we show that the latter is the case.  Throughout this section, we consider only the Gaussian kernel.

\subsection{A curse of dimensionality}\label{sec:curse_of_dimensionality}

We first develop a result that highlights the challenges of inference on blurred TV in high dimensions. 
For any $C\geq 1$, let $\mathcal{P}_d(C)\subseteq\mathcal{P}_d$ be the set of all distributions that are supported on the unit ball $\mathbb{B}_d = \{x\in\R^d:\|x\|_2\leq 1\}$ and have a density bounded by $C(\textnormal{Vol}(\mathbb{B}_d))^{-1}$. (In particular, for $C=1$, the set $\mathcal{P}_d(1)$ contains only the uniform distribution on $\mathbb{B}_d$.)
\begin{theorem}\label{thm:hardness_lowdim_highdim}
    Fix any dimension $d\geq 1$ and sample sizes $n,m\geq 1$.  Let $\psi$ be the Gaussian kernel~\eqref{kernel:gaussian}, and let $\hat{U}_\alpha$ be any DF-UCB on the blurred TV distance $\dtv^h$, for some $h>0$. 
    
    Fix any $\epsilon,\delta\in(0,1)$ and $C\geq 1$. Then there exist constants $c_0,c_1,c_2$ depending only on $C,\epsilon,\delta$, such that if
    \begin{equation}\label{eqn:h_upper_bound_prop:hardness_lowdim_highdim}h \leq c_0 \cdot \begin{cases} \frac{1}{(n\wedge m)^{2/d}\sqrt{d}}, & \textnormal{ if }d\leq c_1\log (n\wedge m),\\
    \frac{1}{\sqrt{\log (n\wedge m)}}, &\textnormal{ if }   d>c_1 \log (n\wedge m) ,\end{cases}\end{equation}
    and $n,m\geq c_2$, then for any $P,Q\in\mathcal{P}_d(C)$,
    \[\PP{\hat{U}_\alpha(\mathcal{D}_{n,m},\zeta) \geq 1-\epsilon} \geq 1-\alpha - \delta.\]
\end{theorem}

Thus, for small $h$, a meaningful DF-UCB is impossible: even if $P=Q$, any valid DF-UCB will be close to $1$ with $\approx 1-\alpha$ probability (similar to the result of Theorem~\ref{thm:hardness}).
On the other hand, if $h$ is sufficiently large, then providing a DF-UCB is a trivial problem for any choice of kernel $\psi$. For the Gaussian kernel, for any $P,Q$ supported on $\mathbb{B}_d$, 
$\dtv^h(P,Q) \leq \frac{\sqrt{2/\pi}}{h}$.
Consequently, for $h\gg 1$, $\dtv^h(P,Q)$ (for bounded $P,Q$) is trivial to bound, since it must necessarily be close to zero.

This result is quite expected given the well-documented challenges posed by high dimensionality in statistics. Intuitively, as $d$ increases, probability mass gets more spread across the support. Small amounts of blurring therefore have little effect on the empirical distributions. Thus, even when $P=Q$, the empirical blurred TV distance can remain close to one in finite samples, forcing the \dfucb{} to remain close to one as well. Intuitively, this is because for a bandwidth $h$ satisfying~\eqref{eqn:h_upper_bound_prop:hardness_lowdim_highdim}, \emph{any} pair of points from $\{X_1,\dots,X_n,Y_1,\dots,Y_m\}$ will have distance $\gg h$---and therefore, blurring with bandwidth $h$ has little effect, and bounding $\dtv^h$ is nearly as difficult as bounding $\dtv$.

\begin{figure}[t]
    \centering
    \fbox{\input{highdim_illustration_tikz}}
    \caption{A visualization of the results of Section~\ref{sec:curse_of_dimensionality}, for distributions $P,Q$ with bounded density on the unit ball.}
    \label{fig:highdim_illustration_tikz}
\end{figure}
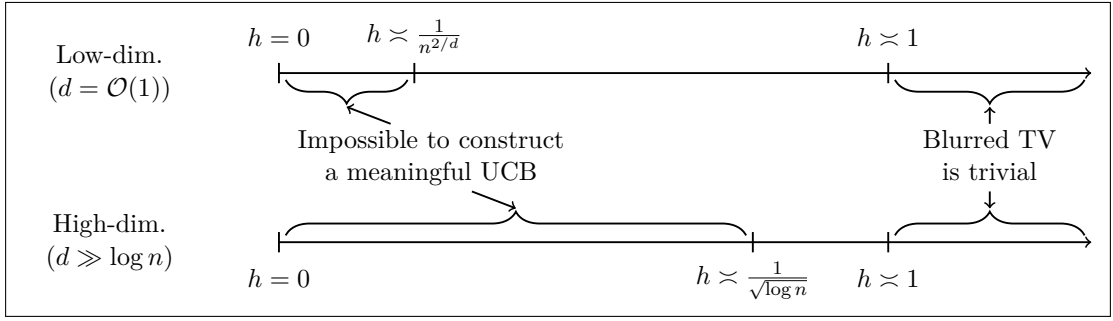

These findings are illustrated in Figure~\ref{fig:highdim_illustration_tikz}. To summarize, if $d=\mathcal{O}(1)$, then the problem of a DF-UCB for blurred TV distance is meaningful across a broad range of bandwidths $h$; there is a wide span of $h$ values in between the regime where the problem is impossible ($h\ll 1/n^{2/d}$) or trivial ($h\gg 1$). However, as soon as $d\gg \log n$, the range of bandwidths where the problem is meaningful becomes much narrower. 

It seems, therefore, that the blurred TV distance suffers from a curse of dimensionality if we wish to perform distribution-free inference. However, there is an important caveat to this conclusion: we have been assuming that the distributions $P,Q$ generating the data have bounded density, which means that $P,Q$ behave similarly to uniform distributions over the unit ball, and do not have latent low-dimensional structure. We turn to this aspect of the question next.

\subsection{Dependence of blurred TV on the effective dimension}
In this section, we show that the blurred TV distance depends only on the \emph{effective dimension} of the underlying distributions, rather than on the ambient dimension $d$. To emphasize the role of dimension, we will write $\psi_d$ for the density of $\mathcal{N}(\boldsymbol{0},I_d)$, and denote blurred TV computed with kernel $\psi_d$ by $\dtv^{h,\psi_d}$.

Suppose $P,Q\in\mathcal{P}_d$ are supported on some $k$-dimensional linear subspace in $\R^d$. The following result shows that the behavior of blurred TV depends on the ``true'' dimension $k$ of these distributions, rather than the ambient dimension $d$. In fact, we analyze a more general setting, where $P$ and $Q$ are approximately supported on such lower dimensional linear subspaces.

\begin{theorem}\label{thm:blurred_tv_depends_on_eff_dimension}
Let $\psi_d$ be the Gaussian kernel~\eqref{kernel:gaussian}. For any $h>0$, and any $P,Q\in\mathcal{P}_d$ and an orthonormal matrix $A\in\R^{d\times k}$,
\[
|\dtv^{h,\psi_d}(P,Q)-\dtv^{h,\psi_k}\bigl(A\circ P, A \circ Q\bigr)|\le \frac{\Ep{X\sim P}{\|X-AA^\top X\|_2}+\Ep{Y\sim Q}{\|Y-AA^\top Y\|_2}}{\sqrt{2\pi}h},
\]
where $A\circ P\in\mathcal{P}_k$ denotes the distribution of $A^\top X$ when $X\sim P$, and similarly for $A\circ Q$.
\end{theorem}

When $P$ and $Q$ are exactly supported on a $k$-dimensional measure whose span is determined by an orthonormal matrix $A\in \R^{d\times k}$, we have 
\[
\dtv^{h,\psi_d}(P,Q)=\dtv^{h,\psi_k}\bigl(A\circ P, A \circ Q\bigr).
\]
Since the empirical measures $\Ph_n$ and $\Qh_m$ are also supported on the same subspace, the same identity holds empirically, i.e., 
$
\dtv^{h,\psi_d}(\Ph_n,\Qh_m)\;=\;\dtv^{h,\psi_k}\bigl(A\circ \Ph_n,\,A \circ \Qh_m\bigr)$.
As a consequence, all confidence bounds developed in earlier sections depend on the intrinsic dimension $k$, rather than the ambient dimension $d$. Such adaptation to intrinsic dimension has appeared before in the related literature; see for example \cite{arias2018remember,cheng2024kernel}.

In practice, it is more likely that a distribution will have approximate, rather than exact, low-dimensional structure: that is, $P$ and $Q$ may be concentrated near some $k$-dimensional subspace. Then, Theorem~\ref{thm:blurred_tv_depends_on_eff_dimension} attests that the true blurred TV between the distributions $P$ and $Q$ lies close to the blurred TV measure computed after doing a lower dimensional projection, provided that the expected projection residuals are small. Therefore, even in such settings, $k$ still behaves as the effective dimension of the problem, allowing us to avoid the curse of dimensionality when we perform distribution-free inference on the blurred TV. We emphasize that the distribution-free guarantees of Section~\ref{sec:confidence_bounds} do not assume any knowledge of the effective dimension: rather, any DF-UCB on the blurred TV will be valid regardless, but can be meaningful (i.e., not close to $1$) only if the effective dimension is low. An interesting open question is whether the result of Theorem~\ref{thm:blurred_tv_depends_on_eff_dimension} can be extended to the more general setting where $P$ and $Q$ are (approximately) supported on a low-dimensional manifold, rather than requiring the support to be a subspace.

\subsection{Simulation study: blurred TV in high dimension}\label{sec:simulation_dimension}
We illustrate how the problem of empirically estimating blurred TV is governed by the effective dimension rather than the ambient dimension, even if the distributions are not exactly supported on a low-dimensional subspace. Define distributions
\[P = \mathcal{N}\big( \beta\cdot \mathbf{e}_1,  (1-\tau)\mathbf{e}_1\mathbf{e}_1^\top + \tau I_d\big),\quad Q = \mathcal{N}\big( -\beta\cdot \mathbf{e}_1,  (1-\tau)\mathbf{e}_1\mathbf{e}_1^\top + \tau I_d\big),
\]
in dimension $d=20$, where $\mathbf{e}_1=(1,0,\dots,0)$.
The parameter $\beta\geq 0$ controls signal strength, with $\beta=0$ corresponding to $P=Q$, while larger $\beta$ increases separation.
The parameter $\tau\in(0,1]$ controls effective dimension, with $\tau\approx 0$ corresponding to effective dimension $\approx 1$.

Figure~\ref{fig:blurred_tv_in_highD} shows the empirical blurred TV $\dtv^h(\Ph_n,\Qh_m)$, computed via Monte Carlo approximation with $B=5000$ from $n=m=200$ samples of each distribution, for various $\beta$ and $\tau$. When the effective dimension is low ($\tau\approx 0$), curves corresponding to different $\beta$ separate clearly: blurred TV decays rapidly for small $\beta$ but remains positive over a wider range of $h$ for larger $\beta$. In contrast, at high effective dimension (large $\tau$), the curves collapse and decay at nearly the same bandwidth, indicating loss of empirical distinguishability between $P=Q$ and $P\neq Q$.

\begin{figure}[t]
    \centering
    \includegraphics[width=0.89\linewidth]{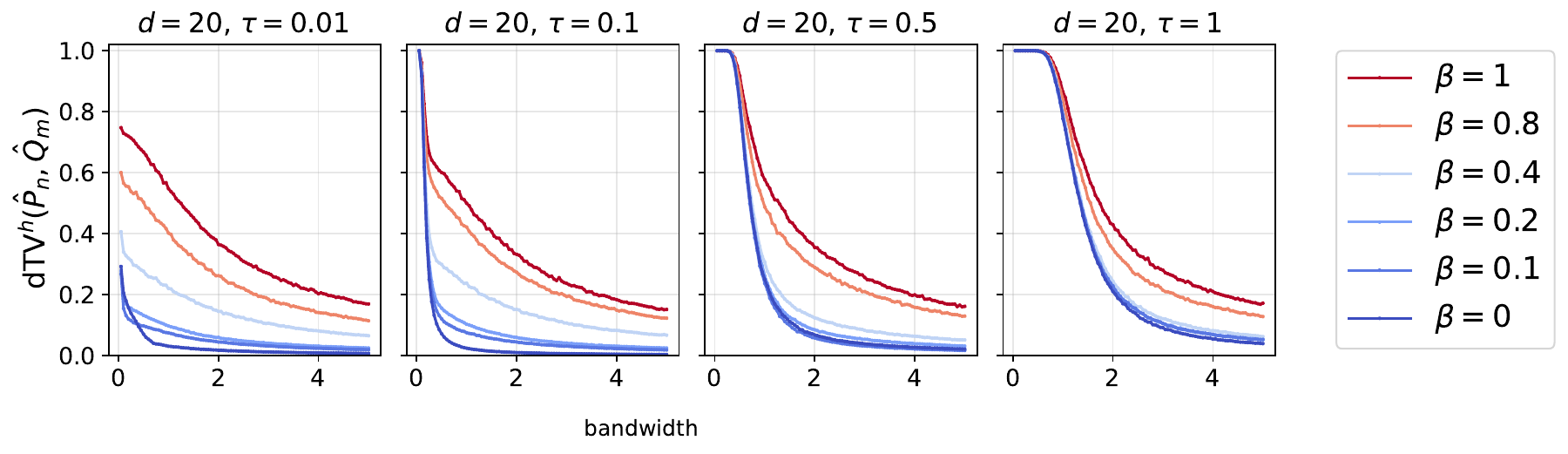}
    \caption{Effect of dimension on the empirical blurred TV $\dtv^h(\Ph_n,\Qh_m)$. See Section~\ref{sec:simulation_dimension} for simulation details.}
    \label{fig:blurred_tv_in_highD}
\end{figure}

\section{Discussion}\label{sec:discussion}
In this work, we have established distribution-free confidence bounds for the blurred TV distance, providing a useful proxy for the classical TV distance (which cannot be meaningfully bounded in an assumption-free regime). Our results are valid regardless of dimension and sample size, and our theoretical and empirical results show that the feasibility of meaningful inference on the blurred TV depends on intrinsic rather than ambient dimension: low effective dimension permits meaningful discrimination even in high-dimensional settings, while genuinely high-dimensional regimes mean that any assumption-free inference on blurred TV will be unable to distinguish whether distributions are similar or not. 

With these challenges in mind, a promising direction for future work would be to develop a meaningful blurred TV measure on low-dimensional projections or functions of high-dimensional data. Moreover, an important open question is whether assumption-free bounds on blurred TV (and any potential extensions for high-dimensional settings) may be useful in modern applications of two-sample testing problems, such as assessing the quality of generative models.

\subsection*{Acknowledgments}
R.F.B. was supported by the Office of Naval Research via grant N00014-24-1-2544 and by the National Science Foundation via grant DMS-202310.

\bibliographystyle{chicago}
\bibliography{reference}

\appendix

\section{Proofs of results from Section~\ref{sec:intro}}
\subsection{Proof of Theorem~\ref{thm:hardness}: impossibility of \dfucb on $\dtv$}
Fix $P,Q$ such that $\textnormal{atom}(P)\cap \textnormal{atom}(Q)=\varnothing$. Let $\hat{U}_\alpha\in [0,1]$ be any DF-UCB on the total variation distance $\dtv(P,Q)$. Now, we consider $\tilde{X}_1,\dots,\tilde{X}_N\iidsim P$, i.e., $N$ samples from distribution $P$ on $\R^d$, and let $\Ph_N$ be the corresponding empirical distribution. If $\tilde{X}_i \not\in\textnormal{atom}(Q)$ for all $i\in [N]$, then $\dtv(\Ph_N,Q)=1$, and therefore by Definition~\ref{def:DF_confidence_bounds},
\begin{multline*}
\Ppst{(\Ph_N)^n\times Q^m\times\textnormal{Unif}[0,1]}{\hat{U}_\alpha((X_1,\dots,X_n),(Y_1,\dots,Y_m),\zeta) = 1}{\tilde{X}_1,\dots,\tilde{X}_N} \\\geq (1-\alpha)\cdot\One{\{\tilde{X}_1,\dots,\tilde{X}_N\}\cap \textnormal{atom}(Q) = \varnothing}.
\end{multline*}
That is, $\hat{U}_\alpha$ must provide a valid upper confidence bound when data is drawn from distributions $\Ph_N$ and $Q$.
We can equivalently write this as follows. Conditionally on $\tilde{X}_1,\dots,\tilde{X}_N$, let $i_1,\dots,i_n\iidsim \textnormal{Unif}([N])$ be indices sampled uniformly at random. Then, $\tilde{X}_{i_1},\dots,\tilde{X}_{i_n}$ are i.i.d.\ samples from $\Ph_N$, and so it holds that 
\begin{multline*}
\Ppst{(\textnormal{Unif}([N]))^n\times Q^m\times\textnormal{Unif}[0,1]}{\hat{U}_\alpha((\tilde{X}_{i_1},\dots,\tilde{X}_{i_n}),(Y_1,\dots,Y_m),\zeta) = 1}{\tilde{X}_1,\dots,\tilde{X}_N}\\\geq (1-\alpha)\cdot\One{\{\tilde{X}_1,\dots,\tilde{X}_N\}\cap \textnormal{atom}(Q) = \varnothing}.
\end{multline*}
 Since $\textnormal{atom}(P)\cap \textnormal{atom}(Q) = \varnothing$, we therefore have $\{\tilde{X}_1,\dots,\tilde{X}_N\}\cap \textnormal{atom}(Q) = \varnothing$ almost surely, and so
\[\Ppst{(\textnormal{Unif}([N]))^n\times Q^m\times\textnormal{Unif}[0,1]}{\hat{U}_\alpha((\tilde{X}_{i_1},\dots,\tilde{X}_{i_n}),(Y_1,\dots,Y_m),\zeta) = 1}{\tilde{X}_1,\dots,\tilde{X}_N}\geq1-\alpha\]
holds almost surely.
After marginalizing over the draw of the data points $\tilde{X}_1,\dots,\tilde{X}_N$, then, we have
\[\Pp{P^N\times (\textnormal{Unif}([N]))^n\times Q^m\times\textnormal{Unif}[0,1]}{\hat{U}_\alpha((\tilde{X}_{i_1},\dots,\tilde{X}_{i_n}),(Y_1,\dots,Y_m),\zeta) = 1}\geq 1-\alpha.\]
Lastly, note that the data set $(\tilde{X}_{i_1},\dots,\tilde{X}_{i_n})$ is generated by sampling $N$ i.i.d.\ draws from $P$, and then sampling $n$ times \emph{with replacement} from this data set; if $N\gg n$ then this is nearly equivalent to simply sampling $n$ i.i.d.\ draws from $P$, and in particular,
\[\dtv\big((\tilde{X}_{i_1},\dots,\tilde{X}_{i_n}),(X_1,\dots,X_n)\big) \leq \frac{n(n-1)}{2N}\]
(see, e.g., \citet[Lemma 4.16]{angelopoulos2024theoretical}). Therefore, it follows that 
\[\Pp{P^n\times Q^m\times \textnormal{Unif}[0,1]}{\hat{U}_\alpha((X_1,\dots,X_n),(Y_1,\dots,Y_m),\zeta) = 1}\geq 1-\alpha - \frac{n(n-1)}{2N}.\]
Finally, since $N$ can be taken to be arbitrarily large, this completes the proof. $\hfill\square$

\subsection{Proof of Proposition~\ref{prop:blurred_TV_metric}}
\label{app:proof_of_blurred_TV_metric}
By definition of TV distance, $\dtv^h(P,P)=0$ for every $P\in\mathcal{P}_d$. Also, since TV distance is a proper metric, $\dtv^h$ is nonnegative, symmetric, and satisfies the triangle inequality. Thus, $\dtv^h$ is always a pseudometric.

For the second part, suppose that $\dtv^h(P,Q)=0$ for some $P,Q\in\mathcal{P}_d$. Then, since TV is a proper metric, $P\ast\psi_h=Q\ast\psi_h$, and therefore
\[
\varphi_P(t)\varphi_\psi(ht)=\varphi_Q(t)\varphi_\psi(ht), \qquad t\in\R^d,
\] 
where $\varphi_P,\varphi_Q$, and $\varphi_\psi$ denote the characteristic functions of $P,Q$, and $\psi$, respectively. Thus, $\varphi_P(t)=\varphi_Q(t)$ whenever $\varphi_\psi(ht)\ne0$. Since $\varphi_\psi$ is nonzero on a dense subset of $\R^d$ and $h>0$, the set $\{t\in\R^d:\varphi_P(t)=\varphi_Q(t)\}$ forms a dense subset of $\R^d$. Continuity of characteristic functions consequently gives $\varphi_P=\varphi_Q$ everywhere, which further yields $P=Q$, as required. \hfill$\square$
\section{Proofs of results from Section~\ref{sec:prop_blurred_TV}}
\subsection{Proof of Proposition~\ref{prop:blurred_tv&tv_properties}: properties of blurred TV}

Throughout the proof, let $X,Y,\xi$ be independent random variables with
$X\sim P$, $Y\sim Q$, and $\xi\sim\psi$.
We will repeatedly use the variational representation of total variation,
\begin{equation}\label{eqn:TV_as_sup}
\dtv(P_0,P_1)=\sup_{f:\R^d\to[0,1]}\bigl|\Ep{W\sim P_0}{f(W)}-\Ep{W\sim P_1}{f(W)}\bigr|,
\end{equation}
as well as its restriction to continuous test functions,
\begin{equation}\label{eqn:TV_as_sup_continuous}
\dtv(P_0,P_1)=\sup_{\textnormal{continuous }f:\R^d\to[0,1]}\bigl|\Ep{W\sim P_0}{f(W)}-\Ep{W\sim P_1}{f(W)}\bigr|,
\end{equation}
which holds for all distributions $P_0,P_1$ on $\R^d$
\citep[e.g.][Lemma~2.3]{larsson2026complete}.

\paragraph{Proof of part (i).}
Fix any function $f:\R^d\to[0,1]$, and define
\[
g(x):=\EE{f(x+h\xi)}.
\]
Then, we have
\begin{align*}
\bigl|\EE{f(X+h\xi)}-\EE{f(Y+h\xi)}\bigr|&=\bigl|\EE{\EEst{f(X+h\xi)}{X}}-\EE{\EEst{f(Y+h\xi)}{Y}}\bigr|\\&=\bigl|\EE{g(X)}-\EE{g(Y)}\bigr|\\&\le \dtv(P,Q),
\end{align*}
where the final inequality follows from~\eqref{eqn:TV_as_sup}.
Since this bound holds for every $f:\R^d\to[0,1]$, applying~\eqref{eqn:TV_as_sup}
once more time yields
\begin{align*}
\dtv^h(P,Q)&=\dtv(P\ast\psi_h,Q\ast\psi_h)\\&=\sup_{f:\R^d\to[0,1]}\bigl|\EE{f(X+h\xi)}-\EE{f(Y+h\xi)}\bigr|\le \dtv(P,Q).
\end{align*}

\paragraph{Proof of part (ii).}
By part~(i), we already have $\dtv^h(P,Q)\le \dtv(P,Q)$.
It therefore suffices to show that for any fixed $\epsilon>0$,
there exists $h_*>0$ such that
\[
\dtv^h(P,Q)\ge \dtv(P,Q)-\epsilon\qquad\textnormal{for all } h\in(0,h_*).
\]
By~\eqref{eqn:TV_as_sup_continuous}, we may choose a continuous function
$f:\R^d\to[0,1]$ satisfying
\[
\dtv(P,Q)\le\bigl|\EE{f(X)}-\EE{f(Y)}\bigr|+\epsilon/2.
\]
By continuity and boundedness of $f$,
\begin{align*}
\lim_{h\to0}\bigl|\EE{f(X)}-\EE{f(X+h\xi)}\bigr|&\le\lim_{h\to0}\EE{\bigl|f(X)-f(X+h\xi)\bigr|}\\&=\EE{\lim_{h\to0}\bigl|f(X)-f(X+h\xi)\bigr|}=0,
\end{align*}
where the interchange of limit and expectation follows from the dominated
convergence theorem.
An identical argument yields
\[
\lim_{h\to0}\bigl|\EE{f(Y)}-\EE{f(Y+h\xi)}\bigr|=0.
\]
Consequently, there exists $h_*>0$ such that for all $h\in(0,h_*)$,
\[
\bigl|\EE{f(X)}-\EE{f(X+h\xi)}\bigr|+\bigl|\EE{f(Y)}-\EE{f(Y+h\xi)}\bigr|\le \epsilon/2.
\]
By the triangle inequality, we further have
\begin{multline*}
\bigl|\EE{f(X)}-\EE{f(Y)}\bigr|\le\bigl|\EE{f(X+h\xi)}-\EE{f(Y+h\xi)}\bigr|\\+
\bigl|\EE{f(X)}-\EE{f(X+h\xi)}\bigr|+\bigl|\EE{f(Y)}-\EE{f(Y+h\xi)}\bigr|.
\end{multline*}
Combining the previous displays, we obtain that for all $h\in(0,h_*)$,
\[
\dtv(P,Q)\le\bigl|\EE{f(X+h\xi)}-\EE{f(Y+h\xi)}\bigr|+\epsilon.
\]
Finally, by~\eqref{eqn:TV_as_sup_continuous},
\[
\bigl|\EE{f(X+h\xi)}-\EE{f(Y+h\xi)}\bigr|\le\dtv^h(P,Q),
\]
and hence,
\[
\dtv(P,Q)\le \dtv^h(P,Q)+\epsilon\quad\textnormal{for all } h\in(0,h_*),
\]
which completes the proof.

\paragraph{Proof of part (iii).}
We begin by observing that, by~\eqref{eqn:TV_as_sup}, 
\begin{align*}
\dtv^h(P,Q)&= \dtv(P\ast\psi_h,Q\ast\psi_h)\\&= \sup_{f:\R^d\to[0,1]}
\bigl|\EE{f(X+h\xi)}-\EE{f(Y+h\xi)}\bigr|\\&= \sup_{g:\R^d\to[0,1]}\bigl|\EE{g(X/h+\xi)}-\EE{g(Y/h+\xi)}\bigr|,
\end{align*}
where the last step holds by considering $g(x)=f(hx)$ for any function $f:\R^d\to[0,1]$.
Applying Jensen’s inequality and again using~\eqref{eqn:TV_as_sup}, we obtain
\begin{align*}
\dtv^h(P,Q)&\le
\EE{\sup_{g:\R^d\to[0,1]}\bigl|\EEst{g(X/h+\xi)}{X}-\EEst{g(Y/h+\xi)}{Y}\bigr|}\\
&=\EE{\dtv(\delta_{X/h}\ast\psi,\delta_{Y/h}\ast\psi)}\\&\le
\EE{\dtv(\delta_{X/h}\ast\psi,\psi)+\dtv(\delta_{Y/h}\ast\psi,\psi)},
\end{align*}
where the last step follows from the triangle inequality.

By Lemma~\ref{lem:dtv_pointmass} below, since $X/h\to 0$ almost surely as $h\to\infty$, it follows that
\[
\dtv(\delta_{X/h}\ast\psi,\psi)\to 0\quad\textnormal{almost surely},
\]
and the same conclusion holds with $Y$ in place of $X$.
Consequently,
\begin{align*}
\lim_{h\to\infty}\dtv^h(P,Q)\le\EE{\lim_{h\to\infty}\bigl(
\dtv(\delta_{X/h}\ast\psi,\psi)+\dtv(\delta_{Y/h}\ast\psi,\psi)\bigr)}=0,
\end{align*}
where the interchange of limit and expectation is justified by the dominated
convergence theorem, since the total variation distance is uniformly bounded by $1$.

\paragraph{Proof of part (iv).}
Let $(h_k)_{k\ge1}$ be any sequence in $[0,\infty)$ such that $h_k\to h\in[0,\infty)$. We aim to show that
\[
\lim_{k\to\infty}\dtv^{h_k}(P,Q)=\dtv^h(P,Q).
\]
If $h=0$, this follows immediately from part~(ii). Hence, from now on we assume $h>0$. Without loss of generality, we may further assume that $h_k>0$ for all $k$, since this holds for all sufficiently large $k$.

As in the proof of part~(iii), we can write
\[
\dtv^h(P,Q)=\sup_{g:\R^d\to[0,1]}\bigl|\EE{g(X/h+\xi)}-\EE{g(Y/h+\xi)}\bigr|,
\]
and similarly,
\[
\dtv^{h_k}(P,Q)=\sup_{g:\R^d\to[0,1]}\bigl|\EE{g(X/h_k+\xi)}-\EE{g(Y/h_k+\xi)}\bigr|.
\]
Therefore,
\begin{align*}
\left|\dtv^h(P,Q)-\dtv^{h_k}(P,Q)\right|&=\Biggl|\sup_{g:\R^d\to[0,1]}
\bigl|\EE{g(X/h+\xi)}-\EE{g(Y/h+\xi)}\bigr| \\
&\hspace{1.5cm}-\sup_{g:\R^d\to[0,1]}\bigl|\EE{g(X/h_k+\xi)}-\EE{g(Y/h_k+\xi)}\bigr|
\Biggr|\\&\le\sup_{g:\R^d\to[0,1]}\bigl|\EE{g(X/h+\xi)}-\EE{g(X/h_k+\xi)}\bigr|\\
&\hspace{1.5cm}+\sup_{g:\R^d\to[0,1]}\bigl|\EE{g(Y/h+\xi)}-\EE{g(Y/h_k+\xi)}\bigr|\\
&\le\EE{\dtv(\delta_{X/h}\ast\psi,\delta_{X/h_k}\ast\psi)}+\EE{\dtv(\delta_{Y/h}\ast\psi,\delta_{Y/h_k}\ast\psi)},
\end{align*}
where the last inequality follows by the same argument as in part~(iii).

Since $X/h_k\to X/h$ almost surely as $k\to\infty$, Lemma~\ref{lem:dtv_pointmass}, together with the dominated convergence theorem, yields
\[
\lim_{k\to\infty}\EE{\dtv(\delta_{X/h}\ast\psi,\delta_{X/h_k}\ast\psi)}=\EE{\lim_{k\to\infty}\dtv(\delta_{X/h}\ast\psi,\delta_{X/h_k}\ast\psi)}=0,
\]
and the same conclusion holds for the term involving $Y$.
Consequently,
\[
\lim_{k\to\infty}\left|\dtv^h(P,Q)-\dtv^{h_k}(P,Q)\right|=0,
\]
which completes the proof.
\hfill $\square$

\begin{lemma}\label{lem:dtv_pointmass}
Let $\psi\in\mathcal{K}_d$. Then
\[
\lim_{r\to0}\sup_{\substack{x,y\in\R^d\\\|x-y\|_2\le r}}\dtv(\delta_x\ast\psi,\delta_y\ast\psi)=0.
\]
\end{lemma}

\begin{proof}
Firstly, since total variation distance is invariant to translations, we have 
\[
\dtv(\delta_x\ast\psi,\delta_y\ast\psi)=\dtv(\psi,\delta_{y-x}\ast\psi).
\]
Thus, it suffices to show that
\[
\lim_{r\to 0}\sup_{\|x\|_2\le r}\dtv(\psi,\delta_x\ast\psi)=0.
\]

Fix an $\epsilon>0$. Since the class of continuous densities is dense in $\mathcal{K}_d$ with respect to total variation distance, we can choose a continuous density $\psi_1$ such that
\[
\dtv(\psi,\psi_1)\le\epsilon.
\]
Then, for all $x\in\R^d$,
\begin{align*}
\dtv(\psi,\delta_x\ast\psi)&\le
\dtv(\psi_1,\delta_x\ast\psi_1)+\dtv(\psi,\psi_1)+\dtv(\delta_x\ast\psi,\delta_x\ast\psi_1)\\
&=\dtv(\psi_1,\delta_x\ast\psi_1)+2\dtv(\psi,\psi_1)\\&\le \dtv(\psi_1,\delta_x\ast\psi_1)+2\epsilon,
\end{align*}
where the second step again uses translation invariance of total variation distance.

The distribution $\delta_x\ast\psi_1$ has density $y\mapsto\psi_1(y-x)$, and therefore
\[\dtv(\psi_1,\delta_x\ast\psi_1)
=\int_{\R^d}\bigl(\psi_1(y)-\psi_1(y-x)\bigr)_+\,\mathsf{d}y.
\]
Consequently,
\begin{align*}
\lim_{r\to0}\sup_{\|x\|_2\le r}\dtv(\psi_1,\delta_x\ast\psi_1)
&\le\lim_{r\to0}\int_{\R^d}\sup_{\|x\|_2\le r}
\bigl(\psi_1(y)-\psi_1(y-x)\bigr)_+\,\mathsf{d}y\\
&=\int_{\R^d}\lim_{r\to0}\sup_{\|x\|_2\le r}
\bigl(\psi_1(y)-\psi_1(y-x)\bigr)_+\,\mathsf{d}y\\&=0,
\end{align*}
where the last equality follows from continuity of $\psi_1$, and the interchange of limit and integral is justified by the dominated convergence theorem, using the bound
\[
0\le\sup_{\|x\|_2\le r}\bigl(\psi_1(y)-\psi_1(y-x)\bigr)_+\le\psi_1(y),
\]
together with the integrability of $\psi_1$. This completes the proof.
\end{proof}

\subsection{Proof of Proposition~\ref{prop:gaussian_monotonicity}: monotonicity of blurred TV with a Gaussian kernel}
We start with observing that for $0<h_1<h_2<\infty$,
    \[
    \mathcal{N}(0,h_2^2)=\mathcal{N}(0,h_1^2) \ast \mathcal{N}(0,h_2^2-h_1^2).
    \]
    Therefore, defining $h_3= \sqrt{h_2^2-h_1^2}$, for any distribution $P$ we have $P\ast \psi_{h_2}=(P\ast \psi_{h_1}) \ast \psi_{h_3}$. Hence, for any distributions $P$ and $Q$,
    \[
        \dtv^{h_2}(P,Q)=\dtv((P\ast \psi_{h_1}) \ast \psi_{h_3},(Q\ast \psi_{h_1}) \ast \psi_{h_3})\leq \dtv(P\ast \psi_{h_1},Q\ast \psi_{h_1})=\dtv^{h_1}(P,Q),
    \]
    where the inequality holds by part (i) of Proposition~\ref{prop:blurred_tv&tv_properties}.
    $\hfill\square$

\subsection{Proof of Theorem~\ref{thm:blurred_TV_convergence_asymptotic}: convergence of the empirical distribution in blurred TV}

We begin by defining a truncation map $g:\R^d\to\R^d$ that projects any point onto the Euclidean ball of radius $R$:
\[
g(x)=\begin{cases}
x, & \|x\|_2\le R,\\[0.3em]
R\,\dfrac{x}{\|x\|_2}, & \|x\|_2>R.
\end{cases}
\]
Let $Q$ denote the distribution of $g(X)$ when $X\sim P$, and let $\Qh_n$ be the empirical distribution of $g(X_1),\ldots,g(X_n)$. By construction, $\Qh_n$ is the empirical distribution of $n$ i.i.d.\ samples from $Q$.

Next, let $\dw$ denote the $1$-Wasserstein distance. By \citet[Theorem~3.1]{lei2020convergence}, we have
\[
\EE{\dw(\Qh_n,Q)}\le C\,R
\begin{cases}
n^{-1/2}, & d=1,\\
n^{-1/2}\log n, & d=2,\\
n^{-1/d}, & d\ge3,
\end{cases}
\]
where $C$ is a universal constant. In particular, this bound can be relaxed to
\[
\EE{\dw(\Qh_n,Q)}\le C\,R\,n^{-1/\max\{2,d\}}\log n.
\]
Now fix any $\epsilon>0$. Applying Lemma~\ref{lem:dtvh_via_coupling} below, with $P_0=\Qh_n$ and $P_1=Q$, we obtain
\[
\dtv^h(\Qh_n,Q)\le \epsilon + \frac{\dw(\Qh_n,Q)}{r(\epsilon,h,\psi)},
\]
where $r(\epsilon,h,\psi)>0$ depends only on $\epsilon,h,\psi$.
Taking expectations yields
\[
\EE{\dtv^h(\Qh_n,Q)}\le \epsilon + \frac{C\,R\,n^{-1/\max\{2,d\}}\log n}{r(\epsilon,h,\psi)}.
\]
We now relate $\dtv^h(\Ph_n,P)$ to the blurred TV between the truncated distributions. By the triangle inequality,
\begin{align*}
\EE{\dtv^h(\Ph_n,P)}
&\le \EE{\dtv^h(\Qh_n,Q)+\dtv^h(\Ph_n,\Qh_n)+\dtv^h(P,Q)}\\
&\le \EE{\dtv^h(\Qh_n,Q)+\dtv(\Ph_n,\Qh_n)+\dtv(P,Q)},
\end{align*}
where the second inequality follows from Proposition~\ref{prop:blurred_tv&tv_properties}. By construction of $Q$, we have
\[
\dtv(P,Q)=\Pp{P}{\|X\|_2>R}.
\]
Similarly, by construction of $\Qh_n$,
\[
\dtv(\Ph_n,\Qh_n)=\Pp{\Ph_n}{\|X\|_2>R},\qquad\EE{\dtv(\Ph_n,\Qh_n)}=\Pp{P}{\|X\|_2>R}.
\]
Combining these bounds, we obtain
\[
\EE{\dtv^h(\Ph_n,P)}\le
\epsilon+ \frac{C\,R\,n^{-1/\max\{2,d\}}\log n}{r(\epsilon,h,\psi)}+ 2\,\Pp{P}{\|X\|_2>R}.
\]
Letting $n\to\infty$ and noting that, for fixed $\epsilon>0$ and $R>0$,
\[
\lim_{n\to\infty}\frac{C\,R\,n^{-1/\max\{2,d\}}\log n}{r(\epsilon,h,\psi)}
=0,
\]
we conclude that
\[
\lim_{n\to\infty}\EE{\dtv^h(\Ph_n,P)}
\le \epsilon + 2\,\Pp{P}{\|X\|_2>R}.
\]
Since the above bound holds for arbitrary $\epsilon>0$ and $R>0$, it follows that
\[
\lim_{n\to\infty}\EE{\dtv^h(\Ph_n,P)}\le \inf_{\epsilon>0,\;R>0}
\left\{\epsilon + 2\,\Pp{P}{\|X\|_2>R}\right\}=0,
\]
which completes the proof.

\hfill $\square$
\begin{lemma}\label{lem:dtvh_via_coupling}
Fix any $\psi\in\mathcal{K}_d$ and any bandwidth $h>0$. Then for every $\epsilon>0$, there exists a constant
$r(\epsilon,h,\psi)>0$ such that
\[
\dtv^h(P_0,P_1)\le \epsilon + \frac{\dw(P_0,P_1)}{r(\epsilon,h,\psi)},
\]
for all distributions $P_0,P_1$ on $\R^d$.
\end{lemma}

\begin{proof}
First, we note that by Lemma~\ref{lem:dtv_pointmass}
(with $\psi_h$ in place of $\psi$), there exists a constant $r>0$ such that
\[
\dtv(\delta_x\ast\psi_h,\delta_y\ast\psi_h)\le \epsilon\qquad\text{for all } \|x-y\|_2\le r,
\]
where $r$ depends only on $\psi_h$ and $\epsilon$.
Now, we fix arbitrary distributions $P_0$ and $P_1$ on $\R^d$. By the coupling representation of the
$1$-Wasserstein distance, there exists a joint distribution $P$ on $\R^d\times\R^d$ such that
$(X,Y)\sim P$ has marginals $X\sim P_0$ and $Y\sim P_1$, and
\[
\Ep{P}{\|X-Y\|_2}=\dw(P_0,P_1).
\]
Let $\xi\sim\psi$ be independent of $(X,Y)$. Thus, by \eqref{eqn:TV_as_sup},
\[
\dtv^h(P_0,P_1)= \dtv(P_0\ast\psi_h,P_1\ast\psi_h)= \sup_{f:\R^d\to[0,1]}
\bigl|\EE{f(X+h\xi)}-\EE{f(Y+h\xi)}\bigr|
\]
Further, by Jensen's inequality and \eqref{eqn:TV_as_sup},
\begin{align*}
\dtv^h(P_0,P_1)&\le \EE{\sup_{f:\R^d\to[0,1]}
\bigl|\EEst{f(X+h\xi)}{X}-\EEst{f(Y+h\xi)}{Y}\bigr|}\\&= \EE{\dtv(\delta_X\ast\psi_h,\delta_Y\ast\psi_h)}.
\end{align*}
By the definition of $r$ above, we have
\[
\dtv(\delta_X\ast\psi_h,\delta_Y\ast\psi_h)\le \epsilon + \one{\{\|X-Y\|_2>r\}}.
\]
Taking expectations yields
\[
\EE{\dtv(\delta_X\ast\psi_h,\delta_Y\ast\psi_h)}\le \epsilon + \PP{\|X-Y\|_2>r}
\le \epsilon + \frac{\EE{\|X-Y\|_2}}{r},
\]
by Markov's inequality. Finally, since
$\EE{\|X-Y\|_2}=\dw(P_0,P_1)$ by construction of the coupling $P$, the desired bound follows.
\end{proof}

\subsection{Proof of Proposition~\ref{prop:E_inequality}: bounds on expected empirical blurred TV}
We begin by establishing the left inequality.  
Fix a bandwidth $h>0$. 

Next, note that $\Ph_n\ast\psi_h = \frac{1}{n}\sum_{i=1}^n \delta_{X_i}\ast \psi_h$. Therefore, for any measurable set $A\subseteq\R^d$,
\begin{align*}
\EE{(\Ph_n\ast\psi_h)(A)}
&=\frac{1}{n}\sum_{i=1}^n \EE{(\delta_{X_i}\ast\psi_h)(A)}\\
&=\EE{(\delta_X\ast\psi_h)(A)}\textnormal{ where $X\sim P$}\\
&=\EE{\PPst{X + h \xi\in A}{X}} \textnormal{ where $\xi\sim \psi$ (with $\xi\independent X$)}\\
&=\PP{X+h\xi\in A}\\
&=(P\ast\psi_h)(A).\end{align*}
An identical argument shows that $\EE{(\Qh_m\ast\psi_h)(A)}=(Q\ast\psi_h)(A)$. 

Therefore, by Jensen’s inequality, 
\begin{align*}
\EE{\dtv^h(\Ph_n,\Qh_m)}
&=\EE{\sup_{\substack{A\subset \R^d,\\ \text{measurable}}}\left|(\Ph_n\ast\psi_h)(A)-(\Qh_m\ast\psi_h)(A)\right|}\\
&\ge \sup_{\substack{A\subset \R^d,\\ \text{measurable}}}\EE{\left|(\Ph_n\ast\psi_h)(A)-(\Qh_m\ast\psi_h)(A)\right|}\\
&\ge \sup_{\substack{A\subset \R^d,\\ \text{measurable}}}\left|\EE{(\Ph_n\ast\psi_h)(A)-(\Qh_m\ast\psi_h)(A)}\right|\\
&= \sup_{\substack{A\subset \R^d,\\ \text{measurable}}}(P\ast\psi_h)(A)-(Q\ast\psi_h)(A)
= \dtv^h(P,Q).
\end{align*}
This proves the left inequality. In particular, this inequality can be equivalently established by noting that the TV distance is convex is jointly convex in its arguments.

We now turn to the right inequality. By triangle inequality, deterministically,
\begin{align}\label{eq:triangle_ineq}
\dtv^h(\Ph_n,\Qh_m)\le \dtv^h(\Ph_n,P)+\dtv^h(P,Q)+\dtv^h(\Qh_m,Q).
\end{align}
Finally, by Lemma~\ref{lem:bounding_TV_bw_emp_and_orig}, we have 
\[
\EE{\dtv^h(\Ph_n,P)}\le \EE{\dtv^h(\Ph^{(1)}_{n/2},\Ph^{(2)}_{n/2})},\quad \EE{\dtv^h(\Qh_m,Q)}\le \EE{\dtv^h(\Qh^{(1)}_{m/2},\Qh^{(2)}_{m/2})}.
\]
The result then follows by taking an expectation over the inequality in~\eqref{eq:triangle_ineq}.

\hfill$\square$

\begin{lemma}\label{lem:bounding_TV_bw_emp_and_orig}
Under the notation of Section~\ref{sec:in_blurred_TV_convergence}, for $n\ge 2$, 
\[
\frac{1}{4}\EE{\dtv^h(\Ph^{(1)}_{n/2},\Ph^{(2)}_{n/2})}
-\frac{1}{4n}
\le \EE{\dtv^h(\Ph_n,P)}
\le \EE{\dtv^h(\Ph^{(1)}_{n/2},\Ph^{(2)}_{n/2})}.
\]
\end{lemma}

\begin{proof}
Let $n_1=\lfloor n/2\rfloor$ and $n_2=\lceil n/2\rceil$. We start by proving the right inequality. Let $X'_1,\ldots,X'_{n_1}\iidsim P$ be drawn independent of $(X_1,\ldots,X_n)$, and let $\Ph^{\prime(1)}_{n/2}$ denote the corresponding empirical measure.
Conditional on $\Ph_n$, the observations
$X_{n_1+1},\ldots,X_n$ can be viewed as sampled from $\Ph_n$ without replacement. Thus, for any measurable set $A\subseteq\R^d$, we have
\[
\EEst{(\Ph^{\prime(1)}_{n/2}\ast\psi_h)(A)}{\Ph_n}
=(P\ast\psi_h)(A),
\qquad
\EEst{(\Ph^{(2)}_{n/2}\ast\psi_h)(A)}{\Ph_n}
=(\Ph_n\ast\psi_h)(A).
\]
Recall from the proof of Proposition~\ref{prop:E_inequality}, TV distance is jointly convex in its arguments. Hence, by applying Jensen's inequality conditional on $\Ph_n$, we get
\[
\dtv^h(\Ph_n,P)
\le
\EEst{\dtv^h(\Ph^{\prime(1)}_{n/2},\Ph^{(2)}_{n/2})}{\Ph_n}.
\]
Taking expectation over the randomness of $\Ph_n$ yields
\[
\EE{\dtv^h(\Ph_n,P)}
\le
\EE{\dtv^h(\Ph^{\prime(1)}_{n/2},\Ph^{(2)}_{n/2})}
=
\EE{\dtv^h(\Ph^{(1)}_{n/2},\Ph^{(2)}_{n/2})},
\]
which proves the first part.

For the left inequality, we start by claiming that 
\begin{equation}\label{eq:dtvh_dec_in_n}
\EE{\dtv^h(\Ph_{a},P)}~~\text{is non-increasing in $a$}.
\end{equation}
Consider any $a,b\in \N$ such that $a>b$. As for the first part, conditional on $\Ph_a$, $\Ph_b$ is the empirical distribution corresponding to a uniformly selected subset of size $b$ from $X_1,\ldots, X_a$. Therefore, by conditional Jensen's inequality,
\[
\dtv^h(\Ph_{a},P)
\le\EEst{\dtv^h(\Ph_b,P)}{\Ph_a}.
\]
Taking expectation over $\hat{P}_a$ yields $\EE{\dtv^h(\Ph_{a},P)}\le \EE{\dtv^h(\Ph_b,P)}$, which proves the claim.

Next, triangle inequality and~\eqref{eq:dtvh_dec_in_n} gives
\begin{equation}\label{eq:triangle_ineq2}
\begin{aligned}
\EE{\dtv^h(\Ph^{(1)}_{n/2},\Ph^{(2)}_{n/2})}
\le\EE{\dtv^h(\Ph^{(1)}_{n/2},P)}
+\EE{\dtv^h(\Ph^{(2)}_{n/2},P)}\le
2\EE{\dtv^h(\Ph^{(1)}_{n/2},P)}.
\end{aligned}
\end{equation}

Note that $\Ph_n=\frac{n_1}{n}\Ph^{(1)}_{n/2}+\frac{n_2}{n}\Ph^{(2)}_{n/2}$,
and the second half is independent of the first. Applying Jensen's inequality conditional on $X_1,\ldots,X_{n_1}$ yields
\begin{align*}
\EEst{\dtv^h(\Ph_n,P)}{X_1,\ldots,X_{n_1}}
\ge\dtv^h\left(
\frac{n_1}{n}\Ph^{(1)}_{n/2}+\frac{n_2}{n}P,P\right)=
\frac{n_1}{n}\dtv^h(\Ph^{(1)}_{n/2},P).
\end{align*}
Taking expectations and applying \eqref{eq:triangle_ineq2},
we obtain
\[
\EE{\dtv^h(\Ph_n,P)}
\ge
\frac{n_1}{2n}
\EE{\dtv^h(\Ph^{(1)}_{n/2},\Ph^{(2)}_{n/2})}.
\]
Finally, since
\[
\frac{n_1}{2n}\ge \frac14-\frac{1}{4n}
\qquad\text{and}\qquad
\dtv^h(\Ph^{(1)}_{n/2},\Ph^{(2)}_{n/2})\le 1,
\]
the left inequality follows.
\end{proof}

\section{Proofs of results in Section~\ref{sec:confidence_bounds}}

Before presenting the proofs, we introduce notation that will be used throughout this section. Let $\bZ = (X_1,\dots,X_n,Y_1,\dots,Y_m)$, where
\[
(X_1,\ldots,X_n)\iidsim P, \qquad (Y_1,\ldots,Y_m)\iidsim Q.
\]
Independently of $\bZ$, let
\[
(X'_1,\ldots,X'_n)\iidsim P, \qquad (Y'_1,\ldots,Y'_m)\iidsim Q,
\]
and write $\bZ'=(X'_1,\ldots,X'_n,Y'_1,\ldots,Y'_m)$. For $i\in[n]$, define the measure
\[
\Ph^{(i)}_n=\frac{1}{n}\Big(\sum_{k\ne i}\delta_{X_k}+\delta_{X'_i}\Big),
\]
which perturbs $\Ph_n$ by replacing a single data point $X_i$ with $X'_i$, and similarly for $j\in[n+m]\setminus[n]$ define $\Qh^{(j)}_m$ by replacing $Y_{j-n}$ with $Y'_{j-n}$.
Similarly, for $i=1,\ldots,\lfloor n/2\rfloor$ and $j=\lfloor n/2\rfloor+1,\ldots,n$, define
\[
\Ph^{(1,i)}_{n/2}
=\frac{1}{\lfloor n/2\rfloor}\biggl(\,\sum_{k=1,k\ne i}^{\lfloor n/2\rfloor}\delta_{X_k}+\delta_{X'_i}\biggr),
\qquad
\Ph^{(2,j)}_{n/2}
=\frac{1}{\lceil n/2\rceil}\biggl(\,\sum_{k=\lfloor n/2\rfloor+1,k\ne j}^{n}\delta_{X_k}+\delta_{X'_j}\biggr).
\]
The measures $\Qh^{(1,i)}_{m/2}$ and $\Qh^{(2,j)}_{m/2}$ are defined analogously. Finally, for $i\in[n+m]$, let $\bZ^{(i)}$ denote the data vector obtained from $\bZ$ by replacing $Z_i$ with $Z'_i$.

\subsection{Proof of Theorem~\ref{thm:naive_confidence_bounds}: distribution-free confidence bounds at fixed bandwidth $h$}\label{app:proof_of_dfcb_fixed_h}

We first prove the \dfucb. Define
\[
F(\bZ):=\dtv^h(\Ph_n,\Qh_m).
\]
We show that $F(\bZ)$ has bounded sensitivity to the replacement of a single observation. For $i\in[n]$, replacing $Z_i=X_i$ by $Z'_i=X'_i$ yields the empirical measure $\Ph_n^{(i)}$ in place of $\Ph_n$, so that $F(\bZ^{(i)})=\dtv^h(\Ph_n^{(i)},\Qh_m)$. By the triangle inequality,
\[
|F(\bZ)-F(\bZ^{(i)})|
=\bigl|\dtv^h(\Ph_n,\Qh_m)-\dtv^h(\Ph_n^{(i)},\Qh_m)\bigr|
\le \dtv^h(\Ph_n,\Ph_n^{(i)})\le \frac{1}{n}.
\]
Similarly, for any $i\in[n+m]\setminus[n]$, we have
\[
|F(\bZ)-F(\bZ^{(i)})|\le \frac{1}{m}.
\]
Applying McDiarmid's inequality, we obtain that for any $\epsilon>0$,
\[
\PP{F(\bZ)\leq \EE{F(\bZ)}- \epsilon}
\le \exp\!\left(-\frac{2\epsilon^2}{1/n+1/m}\right),
\]
and therefore,
\[
\PP{\dtv^h(\Ph_n,\Qh_m)\geq \EE{\dtv^h(\Ph_n,\Qh_m)}-
 \epsilon_{n,m,\alpha}}\ge 1-\alpha.
\]
By Proposition~\ref{prop:E_inequality}, $\EE{\dtv^h(\Ph_n,\Qh_m)}$ provides a deterministic upper bound on $\dtv^h(P,Q)$. This establishes the \dfucb.

We now turn to the \dflcb. We define
\[
G(\bZ)
:=\dtv^h(\Ph_n,\Qh_m)
-\dtv^h(\Ph^{(1)}_{n/2},\Ph^{(2)}_{n/2})
-\dtv^h(\Qh^{(1)}_{m/2},\Qh^{(2)}_{m/2}).
\]
We again verify bounded sensitivity. For any $i=1,\ldots,\lfloor n/2\rfloor$,
\begin{align*}
|G(\bZ)-G(\bZ^{(i)})|
&\le \bigl|\dtv^h(\Ph_n,\Qh_m)-\dtv^h(\Ph_n^{(i)},\Qh_m)\bigr|
+ \bigl|\dtv^h(\Ph^{(1)}_{n/2},\Ph^{(2)}_{n/2})
-\dtv^h(\Ph^{(1,i)}_{n/2},\Ph^{(2)}_{n/2})\bigr| \\
&\le \dtv^h(\Ph_n,\Ph_n^{(i)})
+ \dtv^h(\Ph^{(1)}_{n/2},\Ph^{(1,i)}_{n/2}) \\
&\le \frac{1}{n} + \frac{1}{\lfloor n/2\rfloor}=:c_i
\end{align*}
An analogous bound holds for $i=\lfloor n/2\rfloor+1,\ldots,n$, with 
\[c_i = \frac{1}{n} + \frac{1}{\lceil n/2\rceil}.\]
And, analogous bounds hold for the samples from $Q$: for $i\in[n+m]\backslash[n]$, we have
\[|G(\bZ)-G(\bZ^{(i)})|\leq c_i \textnormal{ where }c_i = \begin{cases}\frac{1}{m} + \frac{1}{\lfloor m/2\rfloor}, & i=n+1,\dots,n+\lfloor m/2\rfloor,\\
\frac{1}{m} + \frac{1}{\lceil m/2\rceil}, & i=n+\lfloor m/2\rfloor +1 ,\dots, n+m.\end{cases}
\]
A straightforward calculation verifies that\[
\sum_{i=1}^{n+m} c_i^2\le \frac{9}{n-1} + \frac{9}{m-1}.
\]
Applying McDiarmid's inequality, we obtain that for any $\epsilon>0$,
\[
\PP{G(\bZ)\geq\EE{G(\bZ)}+ \epsilon}
\le \exp\!\left(-\frac{2\epsilon^2}{9/(n-1)+9/(m-1)}\right),
\]
and therefore,
\begin{equation}\label{eq:G_conc_EG}
    \PP{G(\bZ)\leq\EE{G(\bZ)}+
 3\,\epsilon_{n-1,m-1,\alpha}}
\ge 1-\alpha .
\end{equation}
By Proposition~\ref{prop:E_inequality}, $\EE{G(\bZ)}$ yields a deterministic lower bound on $\dtv^h(P,Q)$. This completes the proof of the \dflcb.

\subsection{Proof of Theorem~\ref{thm:monte_carlo_ucb}: distribution-free validity of Monte Carlo confidence bounds}\label{app:proof_of_monte_carlo_confidence_bounds}
The following result extends Theorem~\ref{thm:monte_carlo_ucb} to provide both a \dfucb and a \dflcb for the Monte Carlo setting.

\begin{reptheorem}{thm:monte_carlo_ucb}
Fix any $d\geq 1$, $n,m\geq 1$, $\psi\in\mathcal{K}_d$, $h>0$, and $\alpha\in(0,1)$. Define
\[\hat{U}_\alpha(\mathcal{D}_{n,m},B)
:= \widehat{\dtv^h}(\Ph_n,\Qh_m;B)
+ \epsilon_{n,m,\alpha/2}
+ \epsilon_{B,\alpha/2}\]
and
\begin{multline*}
\hat{L}_\alpha(\mathcal{D}_{n,m},B)
:= \max\Bigl\{\widehat{\dtv^h}(\Ph_n,\Qh_m;B)
- \widehat{\dtv^h}(\Ph_{n/2}^{(1)},\Ph_{n/2}^{(2)};B)\\
- \widehat{\dtv^h}(\Qh_{m/2}^{(1)},\Qh_{m/2}^{(2)};B)
- 3\,\epsilon_{n-1,m-1,\alpha/2}
- \sqrt{3}\epsilon_{B,\alpha/2},\,0\Bigr\},
\end{multline*}
where $\epsilon_{B,\alpha}:=\left(\frac{\log(1/\alpha)}{2B}\right)^{1/2}$. Then $\hat{U}_\alpha(\mathcal{D}_{n,m},B)$ is a \dfucb for $\dtv^h(\cdot,\cdot)$ at confidence level $1-\alpha$, and $\hat{L}_\alpha(\mathcal{D}_{n,m},B)$ is a \dflcb for $\dtv^h(\cdot,\cdot)$ at confidence level $1-\alpha$.
\end{reptheorem}
Here $\widehat{\dtv^h}(\Ph_n,\Qh_m;B)$ is defined as in Section~\ref{sec:monte_carlo_approximation}, and for the \dflcb, we also define
\[
\widehat{\dtv^h}(\Ph_{n/2}^{(1)},\Ph_{n/2}^{(2)};B)
:=\frac{1}{2B}\sum_{k=1}^B\left|
\frac{\mathsf{d}(\Ph_{n/2}^{(1)} \ast \psi_h) - \mathsf{d}(\Ph_{n/2}^{(2)} \ast \psi_h)}
{\mathsf{d}(\frac{1}{2}\Ph_{n/2}^{(1)}\ast \psi_h+\frac{1}{2}\Ph_{n/2}^{(2)}\ast \psi_h)}\bigl(W'_k + h\xi'_k\bigr)\right|,
\]
where $W'_1,\ldots,W'_B \iidsim \frac{1}{2}\Ph_{n/2}^{(1)}+\frac{1}{2}\Ph_{n/2}^{(2)}$ and $\xi'_1,\ldots,\xi'_B \iidsim \psi$, and
\[
\widehat{\dtv^h}(\Qh_{m/2}^{(1)},\Qh_{m/2}^{(2)};B)
:=\frac{1}{2B}\sum_{k=1}^B\left|
\frac{\mathsf{d}(\Qh_{m/2}^{(1)} \ast \psi_h) - \mathsf{d}(\Qh_{m/2}^{(2)} \ast \psi_h)}
{\mathsf{d}(\frac{1}{2}\Qh_{m/2}^{(1)} \ast \psi_h+\frac{1}{2}\Qh_{m/2}^{(2)} \ast \psi_h)}\bigl(W''_k + h\xi''_k\bigr)\right|,
\]
where $W''_1,\ldots,W''_B \iidsim \frac{1}{2}\Qh_{m/2}^{(1)}+\frac{1}{2}\Qh_{m/2}^{(2)}$ and $\xi''_1,\ldots,\xi''_B \iidsim \psi$, where we draw each sample $\{(W_k,\xi_k)\}$, $\{(W'_k,\xi'_k)\}$, and $\{(W''_k,\xi''_k)\}$ independently.
\begin{proof}
We begin by defining $\bW=(W_1,\ldots,W_B)$ and $\bxi=(\xi_1,\ldots,\xi_B)$, and by writing
\begin{align*}
F(\bW,\bxi)
&:= \widehat{\dtv^{h}}(\Ph_n,\Qh_m;B)\\
&= \frac{1}{2B}
\sum_{k=1}^B\left|
\frac{\mathsf{d}(\Ph_n \ast \psi_{h}) - \mathsf{d}(\Qh_m \ast \psi_{h})}
{\mathsf{d}\bigl(\frac{1}{2}\Ph_n \ast \psi_h + \frac{1}{2}\Qh_m \ast \psi_h\bigr)}
\bigl(W_k + h\xi_k\bigr)
\right|.
\end{align*}
For any $u\in\R^d$, we observe that
\[
\frac{\mathsf{d}(\Ph_n \ast \psi_{h}) - \mathsf{d}(\Qh_m \ast \psi_{h})}
{\mathsf{d}\bigl(\frac{1}{2}\Ph_n \ast \psi_h + \frac{1}{2}\Qh_m \ast \psi_h\bigr)}(u)
\le 2.
\]
It follows that $F(\bW,\bxi)$ satisfies a bounded difference property. Indeed, replacing any single coordinate $(W_i,\xi_i)$ by an independent draw $(W'_i,\xi'_i)\sim (\frac{1}{2}\Ph_n+\frac{1}{2}\Qh_m)\times \psi$ changes the value of $F(\bW,\bxi)$ by at most $1/B$. Applying McDiarmid's inequality conditional on $\mathcal{D}_{n,m}$ therefore yields
\[
\PPst{\widehat{\dtv^{h}}(\Ph_n,\Qh_m;B)
\ge\EEst{\widehat{\dtv^{h}}(\Ph_n,\Qh_m;B)}{\mathcal{D}_{n,m}}
 -\left(\frac{\log(2/\alpha)}{2B}\right)^{1/2}}{\mathcal{D}_{n,m}}
\ge 1-\alpha/2.
\]
Consequently, since $\EEst{\widehat{\dtv^{h}}(\Ph_n,\Qh_m;B)}{\mathcal{D}_{n,m}}=\dtv^h(\Ph_n,\Qh_m)$ by construction, with probability at least $1-\alpha/2$,
\[
\dtv^{h}(\Ph_n,\Qh_m)
\le \widehat{\dtv^{h}}(\Ph_n,\Qh_m;B)
+ \left(\frac{\log(2/\alpha)}{2B}\right)^{1/2}.
\]
Combining this bound with Theorem~\ref{thm:naive_confidence_bounds}, applied at confidence level $1-\alpha/2$, and using a union bound, we obtain
\[
\PP{\dtv^{h}(P,Q)\le \hat{U}_\alpha(\mathcal{D}_{n,m},B)}\ge 1-\alpha.
\]
This establishes the validity of the upper confidence bound.

We now turn to the lower bound.
The corresponding Monte Carlo approximations satisfy an analogous bounded difference property (where now we have $3B$ independent samples, $\{(W_k,\xi_k)\}$ and $\{(W'_k,\xi'_k)\}$ and $\{(W''_k,\xi''_k)\}$, each corresponding to a bounded difference of at most $1/B$). Thus, by McDiarmid's inequality, we obtain that with probability at least $1-\alpha/2$,
\begin{multline*}
\dtv^h(\Ph_n,\Qh_m)
- \dtv^h(\Ph_{n/2}^{(1)},\Ph_{n/2}^{(2)})
- \dtv^h(\Qh_{m/2}^{(1)},\Qh_{m/2}^{(2)})\\
\ge \widehat{\dtv^h}(\Ph_n,\Qh_m;B,h)-
\widehat{\dtv^h}(\Ph_{n/2}^{(1)},\Ph_{n/2}^{(2)};B,h)
- \widehat{\dtv^h}(\Qh_{m/2}^{(1)},\Qh_{m/2}^{(2)};B,h)
- \left(\frac{3\log(2/\alpha)}{2B}\right)^{1/2}.
\end{multline*}
Combining this bound with Theorem~\ref{thm:naive_confidence_bounds}, again applied at confidence level $1-\alpha/2$, applying a union bound yields
\[
\PP{\dtv^{h}(P,Q)\ge \hat{L}_\alpha(\mathcal{D}_{n,m},B)}\ge 1-\alpha.
\]
This completes the proof.
\end{proof}

\subsection{Proof of Theorem~\ref{thm:ucb_uniform}: distribution-free validity of confidence bounds uniformly over all bandwidths}\label{app:proof_of_uniform_confidence_bounds}
Here we state and prove a \dfucb and \dflcb on blurred TV that are valid uniformly over a range of bandwidths.
\begin{reptheorem}{thm:ucb_uniform}
Fix any $d\geq 1$, $n,m\geq 1$, $\psi\in\mathcal{K}_d$, and $\alpha\in(0,1)$. For all $h>0$, define 
\[\hat{U}^{h,\mathrm{up}}_\alpha(\mathcal{D}_{n,m})
:= \dtv^{h,\mathrm{up}}(\Ph_n,\Qh_m)
+ \epsilon_{n,m,\alpha/(n\wedge m)}+\frac{1}{n\wedge m}\]
and, for any fixed $h^*>0$, for all $h\geq h^*$ define
\begin{multline*}
\hat{L}^{h,\mathrm{lo}}_\alpha(\mathcal{D}_{n,m})
:= \max\Bigl\{\,\inf_{h'\in[h^*,h]}\bigl\{\dtv^{h'}(\Ph_n,\Qh_m)
- \dtv^{h'}(\Ph^{(1)}_{n/2},\Ph^{(2)}_{n/2})
- \dtv^{h'}(\Qh^{(1)}_{m/2},\Qh^{(2)}_{m/2})\bigr\} \\
- 3\,\epsilon_{n-1,m-1,\alpha/(n\wedge m)} - \frac{1}{n\wedge m},\,0\Bigr\}.
\end{multline*}
Then $\hat{U}^{h,\mathrm{up}}_\alpha(\mathcal{D}_{n,m})$ is a \dfucb for $\dtv^h(\cdot,\cdot)$, uniformly for all $h\in(0,\infty)$, at confidence level $1-\alpha$, and $\hat{L}^{h,\mathrm{lo}}_\alpha(\mathcal{D}_{n,m})$ is a \dflcb for $\dtv^h(\cdot,\cdot)$, uniformly for all $h\in[h^*,\infty)$, at confidence level $1-\alpha$.
\end{reptheorem}
\begin{proof}
First we consider the \dfucb. Fix any distributions $P,Q\in\mathcal{P}_d$. We can assume $\dtv(P,Q)>\frac{1}{n\wedge m}$ since otherwise the result holds trivially. Let $k\in\{1,\dots,(n\wedge m)-1\}$ be the largest integer such that $\dtv(P,Q) > \frac{k}{n\wedge m}$. For each $i=1,\dots,k$, define
\[h_i = \sup\left\{h\in[0,\infty) : \dtv^h(P,Q) \geq \frac{i}{n\wedge m}\right\}.\]
Recalling Proposition~\ref{prop:blurred_tv&tv_properties}, we can see that $h_i$ is finite and positive, and $\dtv^{h_i}(P,Q)=i/(n\wedge m)$, for each $i$.
By Theorem~\ref{thm:naive_confidence_bounds},
\[\dtv^{h_i}(P,Q)\leq \dtv^{h_i}(\Ph_n,\Qh_m)
+ \epsilon_{n,m,\alpha/(n\wedge m)}\]
holds with probability $\geq 1-\alpha/(n\wedge m)$, for each $i$. By a union bound, therefore,
\[\PP{\dtv^{h_i}(P,Q)\leq \dtv^{h_i}(\Ph_n,\Qh_m)
+ \epsilon_{n,m,\alpha/(n\wedge m)}\textnormal{ for all $i=1,\dots,k$}}\geq 1-\alpha.\]
Now suppose this event holds. We will now show that
\[\dtv^h(P,Q) \leq \hat{U}^{h,\mathrm{up}}_\alpha(\mathcal{D}_{n,m})\textnormal{ for all $h\in(0,\infty)$}.\]
Fix any $h$. If $\dtv^h(P,Q)<1/(n\wedge m)$ then this is trivial. Otherwise, let $i = \min\{k,\lfloor (n\wedge m) \cdot \dtv^h(P,Q)\rfloor\}$, so that 
\[\frac{i}{n\wedge m}\leq \dtv^h(P,Q) \leq \frac{i+1}{n\wedge m}.\]
By definition of $h_i$ we must have $h\leq h_i$, and therefore
\[\dtv^{h,\mathrm{up}}(\Ph_n,\Qh_m)\geq \dtv^{h_i}(\Ph_n,\Qh_m) \geq \dtv^{h_i}(P,Q) - \epsilon_{n,m,\alpha/(n\wedge m)},\]
by the event assumed above. But
\[\dtv^{h_i}(P,Q) = \frac{i}{n\wedge m} \geq \dtv^h(P,Q) - \frac{1}{n\wedge m},\]
by definition of $i$, which completes the proof for the \dfucb.

Next we prove the \dflcb with a similar argument.
Let $k\in\{1,\dots,(n\wedge m)-1\}$ be the largest integer such that $\dtv^{h^*}(P,Q) > \frac{k}{n\wedge m}$, or $k=0$ if $\leq \frac{1}{n\wedge m}$. For each $i=0,1,\dots,k$, we now define
\[h_i = \inf\left\{h\in[h^*,\infty) : \dtv^h(P,Q) \leq \dtv^{h^*}(P,Q) - \frac{i}{n\wedge m}\right\}.\]
Note that $h_0=h^*$. By Proposition~\ref{prop:blurred_tv&tv_properties}, we can see that $h_i$ is finite and positive, and $\dtv^{h_i}(P,Q)=\dtv^{h^*}(P,Q)-i/(n\wedge m)$, for each $i$.
By Theorem~\ref{thm:naive_confidence_bounds},
\[\dtv^{h_i}(P,Q)\geq \bigl\{\dtv^{h_i}(\Ph_n,\Qh_m)-\dtv^{h_i}(\Ph_{n/2}^{(1)},\Ph_{n/2}^{(2)})-\dtv^{h_i}(\Qh_{m/2}^{(1)},\Qh_{m/2}^{(2)})\bigr\}
-3\epsilon_{n-1,m-1,\alpha/(n\wedge m)}\]
holds with probability $\geq 1-\alpha/(n\wedge m)$, for each $i=0,1,\dots,k$. By a union bound, therefore,
\begin{multline*}\mathbb{P}\bigg\{\dtv^{h_i}(P,Q)\geq \bigl\{\dtv^{h_i}(\Ph_n,\Qh_m)-\dtv^{h_i}(\Ph_{n/2}^{(1)},\Ph_{n/2}^{(2)})-\dtv^{h_i}(\Qh_{m/2}^{(1)},\Qh_{m/2}^{(2)})\bigr\}\\ - 
3\epsilon_{n-1,m-1,\alpha/(n\wedge m)}\textnormal{ for all $i=0,\dots,k$}\bigg\}\geq 1-\alpha.\end{multline*}
Now suppose the event inside the above display holds. We will show that
\[\dtv^h(P,Q) \geq \hat{L}^{h,\mathrm{lo}}_\alpha(\mathcal{D}_{n,m})\textnormal{ for all $h\in[h^*,\infty)$}.\]
Fix any $h\ge h^\star$. First suppose that $\dtv^h(P,Q)\ge \dtv^{h^\star}(P,Q)$. Since, by definition, $\widehat L_\alpha^{h,\mathrm{lo}}(D_{n,m})$ is nonincreasing in $h$, and the event assumed above includes $h_0=h^\star$, we have
\[\widehat L_\alpha^{h,\mathrm{lo}}(D_{n,m})\le\widehat L_\alpha^{h^\star,\mathrm{lo}}(D_{n,m})\le\dtv^{h^\star}(P,Q)\le\dtv^h(P,Q).\]
It remains to consider the case $\dtv^h(P,Q)<\dtv^{h^\star}(P,Q)$. In this case, there exists $i\in\{0,\ldots,k\}$ such that
\[\dtv^{h^*}(P,Q)-\frac{i+1}{n\wedge m}\leq \dtv^h(P,Q) \leq \dtv^{h^*}(P,Q)-\frac{i}{n\wedge m}.\]
By definition of $h_i$ we must have $h\geq h_i\geq h^*$, and therefore
\begin{align*}
    &\hspace{-1.5em}\inf_{h'\in[h^*,h]}\bigl\{\dtv^{h'}(\Ph_n,\Qh_m)
- \dtv^{h'}(\Ph^{(1)}_{n/2},\Ph^{(2)}_{n/2})
- \dtv^{h'}(\Qh^{(1)}_{m/2},\Qh^{(2)}_{m/2})\bigr\}\\
&\leq \dtv^{h_i}(\Ph_n,\Qh_m)
- \dtv^{h_i}(\Ph^{(1)}_{n/2},\Ph^{(2)}_{n/2})
- \dtv^{h_i}(\Qh^{(1)}_{m/2},\Qh^{(2)}_{m/2})\bigr\}\\
&\leq \dtv^{h_i}(P,Q) + 3\epsilon_{n-1,m-1,\alpha/(n\wedge m)},
\end{align*}
by the event assumed above.
But
\[\dtv^{h_i}(P,Q) = \dtv^{h^*}(P,Q)-\frac{i}{n\wedge m} \leq \dtv^h(P,Q) + \frac{1}{n\wedge m},\]
by definition of $i$, which completes the proof for the \dflcb.
\end{proof}

\subsection{Proof of Theorem~\ref{thm:ucb_bandwidth_adaptive}: bandwidth-adaptive distribution-free confidence bounds}\label{app:proof_of_bandwidth_adaptive_cb}
In this section, we state and prove the full version of Theorem~\ref{thm:ucb_bandwidth_adaptive}, to construct a bandwidth-adaptive \dfucb and \dflcb.
\begin{reptheorem}{thm:ucb_bandwidth_adaptive}
Fix any $d\geq 1$, $n,m\geq 1$, $\psi\in\mathcal{K}_d$, $h>0$, and $\alpha\in(0,1)$. Define 
\[
\hat{U}_\alpha(\mathcal{D}_{n,m},\hat\Sigma_{n,m}^{h})=\dtv^h(\Ph_n,\Qh_m)
+ \bigl(\hat\Sigma_{n,m}^{h}\bigr)^{1/2} r_{1,\alpha}
+ \frac{1}{n\wedge m}
\left(\frac{r_{1,\alpha}^2}{3}
+ 2\sqrt{5}\,r_{1,\alpha} r_{2,\alpha}\right)
\]and
\begin{multline*}
  \hat{L}_\alpha(\mathcal{D}_{n,m},\hat\Sigma_{n,m}^{h})=\dtv^h(\Ph_n,\Qh_m)
- \dtv^h(\Ph^{(1)}_{n/2},\Ph^{(2)}_{n/2})
- \dtv^h(\Qh^{(1)}_{m/2},\Qh^{(2)}_{m/2})\\
- 6\,\bigl(\hat\Sigma_{n,m}^{h}\bigr)^{1/2} r_{1,\alpha}
- \frac{1}{n\wedge m}
\left(2r_{1,\alpha}^2 + 12\sqrt{5}r_{1,\alpha}r_{2,\alpha}\right),
\end{multline*}
where
$r_{1,\alpha}=(2\log(2/\alpha))^{1/2}$ and
$r_{2,\alpha}=(\log(16/\alpha))^{1/2}$.
Then $\hat{U}_\alpha(\mathcal{D}_{n,m},\hat\Sigma_{n,m}^{h})$ is a \dfucb for $\dtv^h(\cdot,\cdot)$ at confidence level $1-\alpha$, and $\hat{L}_\alpha(\mathcal{D}_{n,m},\hat\Sigma_{n,m}^{h})$ is a \dflcb for $\dtv^h(\cdot,\cdot)$ at confidence level $1-\alpha$.
\end{reptheorem}
\begin{proof}
In order to prove the result, we introduce the variance proxy
\[\tilde\Sigma:=\frac{1}{n}\Ep{X,X'}{\omega_\psi\!\left(\frac{X-X'}{h}\right)^2}
+\frac{1}{m}\Ep{Y,Y'}{\omega_\psi\!\left(\frac{Y-Y'}{h}\right)^2},
\]
where $(X,X')\iidsim P$ and $(Y,Y')\iidsim Q$.
Applying Proposition~\ref{prop:E_inequality},
\[
    \dtv^h(P,Q)
    \leq \EE{\dtv^h(\Ph_n,\Qh_m)}
    \leq \dtv^h(\Ph_n,\Qh_m) + \tilde\Sigma^{1/2}r_{1,\alpha} + \frac{r_{1,\alpha}^2}{3(n\wedge m)},
\] where the last step holds 
with probability $\geq 1-\alpha/2$ by Lemma~\ref{lem:Bernstein} below, and moreover by Lemma~\ref{lem:Sigma_concentration} below,
\[\tilde\Sigma^{1/2}\leq \bigl(\hat\Sigma_{n,m}^{h}\bigr)^{1/2}
+\frac{2\sqrt{5}}{n\wedge m}\,r_{2,\alpha}\]
also holds with probability $\geq 1-\alpha/2$. Combining these calculations verifies the \dfucb; a similar argument holds for the \dflcb.
\end{proof}

\begin{lemma}\label{lem:Bernstein}
    In the setting of Theorem~\ref{thm:ucb_bandwidth_adaptive} (extended version), 
    \[\PP{\dtv^h(\Ph_n,\Qh_m)  \geq \EE{\dtv^h(\Ph_n,\Qh_m)}
    -\left( \tilde\Sigma^{1/2}r_{1,\alpha} + \frac{r_{1,\alpha}^2}{3(n\wedge m)}\right)}\geq1-\alpha/2\]
    and
    \begin{multline*}\mathbb{P}\Bigg\{ \dtv^h(\Ph_n,\Qh_m) - \dtv^h(\Ph^{(1)}_{n/2},\Ph^{(2)}_{n/2})
- \dtv^h(\Qh^{(1)}_{m/2},\Qh^{(2)}_{m/2}) \\\leq
\EE{\dtv^h(\Ph_n,\Qh_m)- \dtv^h(\Ph^{(1)}_{n/2},\Ph^{(2)}_{n/2})
- \dtv^h(\Qh^{(1)}_{m/2},\Qh^{(2)}_{m/2})} \\+ \left( 6\tilde\Sigma^{1/2}r_{1,\alpha} + \frac{2r_{1,\alpha}^2}{n\wedge m}\right)\Bigg\}\geq1-\alpha/2.\end{multline*}
\end{lemma}
\begin{proof}
To prove the first claim, let
\[
F(\bZ):=\dtv^h(\Ph_n,\Qh_m).
\]
    For $i\in[n]$, replacing $Z_i=X_i$ by an independent copy $X'_i$ yields the empirical measure $\Ph_n^{(i)}$, so that \[F(\bZ^{(i)})=\dtv^h(\Ph_n^{(i)},\Qh_m).\] By the triangle inequality,
\begin{align}\label{eqn:bounding_difference_in_F}
|F(\bZ)-F(\bZ^{(i)})|
&= \bigl|\dtv^h(\Ph_n,\Qh_m)-\dtv^h(\Ph_n^{(i)},\Qh_m)\bigr|\notag\\
&\le \dtv^h(\Ph_n,\Ph_n^{(i)})
= \frac{1}{n}\dtv(\delta_{(X_i-X'_i)/h}\ast\psi,\psi)
= \frac{1}{n}\,\omega_\psi\!\left(\frac{X_i-X'_i}{h}\right).
\end{align}
Since $\omega_\psi(\cdot)\le1$, it follows that for all $i\in[n]$,
\[
|F(\bZ)-F(\bZ^{(i)})|\le \frac{1}{n}\le\frac{1}{n\wedge m}.
\]
An identical argument gives analogous inequalities for $i\in[n+m]\setminus[n]$. Next define
\[
\tilde\Sigma^*
:=\sum_{i=1}^{n+m}\sup_{\bZ_{-i}}\,
\Ep{Z_i,Z'_i}{\bigl(F(\bZ)-F(\bZ^{(i)})\bigr)^2},
\]
where for each $i$, the supremum is taken over all possible values in the remaining entries $j\in[n+m]\backslash\{i\}$ of the random vector $\bZ$, while the expected value is taken with respect to the distribution (i.e., $Z_i,Z'_i\iidsim P$ if $i\in[n]$, and $Z_i,Z'_i\iidsim Q$ if $i\in[n+m]\backslash[n]$).
Observe that, by~\eqref{eqn:bounding_difference_in_F}, we have $\tilde\Sigma^*\leq \tilde\Sigma$, as follows:
\begin{align*}
\tilde\Sigma^*
&=\sum_{i=1}^{n+m}\sup_{\bZ}\Ep{Z_i,Z'_i}{\bigl(F(\bZ)-F(\bZ^{(i)})\bigr)^2}\\
&\le
\sum_{i=1}^n \sup_{\bZ}\Ep{X_i,X'_i}{\left(\frac{1}{n}\omega_\psi\!\left(\frac{X_i-X'_i}{h}\right)\right)^2}
+\sum_{j=1}^m \sup_{\bZ}\Ep{Y_j,Y'_j}{\left(\frac{1}{m}\omega_\psi\!\left(\frac{Y_j-Y'_j}{h}\right)\right)^2}\\
&=
\frac{1}{n}\Ep{X,X'\iidsim P}{\omega_\psi\!\left(\frac{X-X'}{h}\right)^2}
+\frac{1}{m}\Ep{Y,Y'\iidsim Q}{\omega_\psi\!\left(\frac{Y-Y'}{h}\right)^2}=\tilde\Sigma.
\end{align*}

Applying \citep[Theorem~1]{maurer2019bernstein} to the function $F(\bZ)$, we conclude that for any $\epsilon\ge0$,
\[
\PP{F(\bZ)\leq \EE{F(\bZ)}-\epsilon}
\le \exp\!\left(
-\frac{\epsilon^2}{2\bigl(\tilde\Sigma+\frac{\epsilon}{3(n\wedge m)}\bigr)}
\right).
\]
Therefore, by the definition of $r_{1,\alpha}$,
\begin{equation}\label{eqn:maurer_concentration_ineq}
\PP{
F(\bZ)\leq\EE{F(\bZ)} -
\left( \tilde\Sigma^{1/2}\,r_{1,\alpha}
+\frac{r_{1,\alpha}^2}{3(n\wedge m)}\right)}
\le \frac{\alpha}{2}.
\end{equation}

Next we turn to the second claim. Let
\[G(\bZ):=\dtv^h(\Ph_n,\Qh_m) - \dtv^h(\Ph^{(1)}_{n/2},\Ph^{(2)}_{n/2})
- \dtv^h(\Qh^{(1)}_{m/2},\Qh^{(2)}_{m/2}).\]
We may assume $n\wedge m\geq 2$ (since otherwise this part of the claim is vacuous). A similar calculation as for $F$ shows that
\[
|G(\bZ)-G(\bZ^{(i)})|
\le \frac{3}{n-1}\,\omega_\psi\!\left(\frac{X_i-X'_i}{h}\right)\leq \frac{6}{n}\,\omega_\psi\!\left(\frac{X_i-X'_i}{h}\right)\leq \frac{6}{n\wedge m}\]
for all $i\in[n]$, and similar for $i\in[n+m]\backslash[n]$.
Again applying \citep[Theorem~1]{maurer2019bernstein} we then obtain
\[\PP{
G(\bZ)\geq\EE{G(\bZ)} +
\left( 6\tilde\Sigma^{1/2}\,r_{1,\alpha}
+\frac{2r_{1,\alpha}^2}{n\wedge m}\right)}
\le \frac{\alpha}{2}.\]
\end{proof}

\begin{lemma}\label{lem:Sigma_concentration}
    In the setting of Theorem~\ref{thm:ucb_bandwidth_adaptive} (extended version), 
    \[\PP{\tilde\Sigma^{1/2} \leq \bigl(\hat\Sigma_{n,m}^{h}\bigr)^{1/2}
+\frac{2\sqrt{5}}{n\wedge m}\,r_{2,\alpha}}\geq 1-\alpha/2.\]
\end{lemma}
\begin{proof}
We can assume $n,m\geq 4$, since otherwise the result is vacuous.
We write
\[
\hat{\Sigma}_{n,m}^{h}:=\frac{1}{n}U_X+\frac{1}{m}U_Y,
\]
where
\[
U_X:=\frac{1}{\binom n2}\sum_{1\le i<j\le n}
\omega_\psi\!\left(\frac{X_i-X_j}{h}\right)^2,
\qquad
U_Y:=\frac{1}{\binom m2}\sum_{1\le i<j\le m}
\omega_\psi\!\left(\frac{Y_i-Y_j}{h}\right)^2.
\]
Clearly, by construction,
\[
\EE{U_X}=\Ep{X,X'}{\omega_\psi\!\left(\frac{X-X'}{h}\right)^2},
\qquad
\EE{U_Y}=\Ep{Y,Y'}{\omega_\psi\!\left(\frac{Y-Y'}{h}\right)^2},
\]
and so $\EE{\hat\Sigma^h_{n,m}} = \tilde\Sigma$.

To control the deviation of $U_X$ and $U_Y$ from their expectations, we introduce the empirical variance estimators
\begin{align*}
\sigma_X^2&:=\frac{1}{n(n-1)(n-2)(n-3)}
\sum_{\substack{1\le i,j,k,l\le n\\ i\neq j\neq k\neq l}}
\left[\omega_\psi\!\left(\frac{X_i-X_j}{h}\right)^2
-\omega_\psi\!\left(\frac{X_k-X_l}{h}\right)^2\right]^2,\\
\sigma_Y^2&:=\frac{1}{m(m-1)(m-2)(m-3)}
\sum_{\substack{1\le i,j,k,l\le m\\ i\neq j\neq k\neq l}}
\left[\omega_\psi\!\left(\frac{Y_i-Y_j}{h}\right)^2
-\omega_\psi\!\left(\frac{Y_k-Y_l}{h}\right)^2\right]^2.
\end{align*}
Since $\omega_\psi(\cdot)\in[0,1]$ and $(a-b)^2\le a+b$ for $a,b\in[0,1]$, we obtain
\begin{align*}
\sigma_X^2
&\le\frac{1}{n(n-1)(n-2)(n-3)}
\sum_{\substack{1\le i,j,k,l\le n\\ i\neq j\neq k\neq l}}
\left[\omega_\psi\!\left(\frac{X_i-X_j}{h}\right)^2
+\omega_\psi\!\left(\frac{X_k-X_l}{h}\right)^2\right]\\
&=\frac{1}{n(n-1)}\sum_{1\le i\ne j\le n}\omega_\psi\!\left(\frac{X_i-X_j}{h}\right)^2
+\frac{1}{n(n-1)}\sum_{1\le k\ne l\le n}\omega_\psi\!\left(\frac{X_k-X_l}{h}\right)^2
\;\le\;2U_X,
\end{align*}
and analogously $\sigma_Y^2\le2U_Y$.

We now apply \citep[Theorem~3]{peel2010empirical} to the order-$2$ U-statistics $U_X$ and $U_Y$. By the definition of $r_{2,\alpha}$, with probability at least $1-\alpha/4$,
\[
\bigl|U_X-\EE{U_X}\bigr|
\le\left(\frac{8U_X}{n}\,r_{2,\alpha}^2\right)^{1/2}
+\frac{10}{n}\,r_{2,\alpha}^2,
\]
and similarly, with probability at least $1-\alpha/4$,
\[
\bigl|U_Y-\EE{U_Y}\bigr|
\le
\left(\frac{8U_Y}{m}\,r_{2,\alpha}^2\right)^{1/2}
+\frac{10}{m}\,r_{2,\alpha}^2.
\]
By a union bound, it follows that with probability at least $1-\alpha/2$,
\begin{align*}
\tilde\Sigma
&\le \frac{1}{n}\EE{U_X}+\frac{1}{m}\EE{U_Y}\\
&\le \hat\Sigma_{n,m}^{h}
+\left(\frac{8U_X}{n^3}\right)^{1/2} r_{2,\alpha}
+\left(\frac{8U_Y}{m^3}\right)^{1/2} r_{2,\alpha}
+\frac{20}{(n\wedge m)^2}\,r_{2,\alpha}^2.
\end{align*}
Using $n,m\ge n\wedge m$, we further obtain
\[
\tilde\Sigma
\le
\hat\Sigma_{n,m}^{h}
+\frac{4}{n\wedge m}\,\bigl(\hat\Sigma_{n,m}^{h}\bigr)^{1/2}\,r_{2,\alpha}
+\frac{20}{(n\wedge m)^2}\,r_{2,\alpha}^2.
\]
Consequently, with probability at least $1-\alpha/2$,
\[
\tilde\Sigma^{1/2}
\le \sqrt{\hat\Sigma_{n,m}^{h}
+\frac{4}{n\wedge m}\,\bigl(\hat\Sigma_{n,m}^{h}\bigr)^{1/2}\,r_{2,\alpha}
+\frac{20}{(n\wedge m)^2}\,r_{2,\alpha}^2}\leq
\bigl(\hat\Sigma_{n,m}^{h}\bigr)^{1/2}
+\frac{2\sqrt{5}}{n\wedge m}\,r_{2,\alpha}.
\]

\end{proof}




\section{Proofs of results from Section~\ref{sec:highD_blurred_TV}}

\subsection{Proof of Theorem~\ref{thm:hardness_lowdim_highdim}: the curse of dimensionality}

Without loss of generality, assume $n\leq m$. The key tool for the proof, across both regimes for dimension $d$, is the following lemma:
\begin{lemma}\label{lem:apply_sample_resample}
    Fix any $d\geq 1$. Let $\psi$ be the Gaussian kernel~\eqref{kernel:gaussian}, and let $\hat{U}_\alpha$ be any DF-UCB on the blurred TV distance $\dtv^h$, for some $h>0$. Fix any distributions $P,Q$ on $\R^d$, and suppose that for some sample size $N\geq 1$ and some $\epsilon,\delta'\in(0,1)$ that
    \[\PP{\dtv^h(\Ph_N,Q) \geq 1-\epsilon} \geq 1-\delta',\]
    where $\Ph_N$ is the empirical distribution of $N$ i.i.d.\ draws from $P$. Then, for any $n,m\geq 1$,
    \[\PP{\hat{U}_\alpha(\mathcal{D}_{n,m},\zeta) \geq 1-\epsilon} \geq 1-\alpha - \delta' - \frac{n(n-1)}{2N},\]
    where the probability is calculated with respect to a data set $\mathcal{D}_{n,m}$ comprised of $X_1,\dots,X_n\iidsim P$ and $Y_1,\dots,Y_m\iidsim Q$.
\end{lemma}

We first complete the proof using this lemma, and then prove the lemma in Appendix~\ref{app:lemma_proof_curse_dimensionality}.
First, let $\delta' = \delta/2$ and define
\[N = \left\lceil \frac{n(n-1)}{\delta}\right\rceil.\]
Applying Lemma~\ref{lem:apply_sample_resample}, for any distributions $P,Q$ we see that
\[\textnormal{If }\PP{\dtv^h(\Ph_N,Q) \geq 1-\epsilon} \geq 1-\delta/2, \textnormal{\ then \ } \PP{\hat{U}_\alpha(\mathcal{D}_{n,m},\zeta) \geq 1-\epsilon} \geq 1-\alpha - \delta.\]
Consequently, it suffices to verify that, if $h$ satisfies the assumed upper bound~\eqref{eqn:h_upper_bound_prop:hardness_lowdim_highdim}, then
\begin{multline}\label{eqn:show_N_prop:hardness_lowdim_highdim}
    \textnormal{For any $P,Q\in\mathcal{P}_d(C)$, and for $N = \left\lceil \frac{n(n-1)}{\delta}\right\rceil$,}\\ \textnormal{it holds that $\PP{\dtv^h(\Ph_N,Q) \geq 1-\epsilon} \geq 1-\delta/2$}.
\end{multline}
We will verify this separately for the two cases of~\eqref{eqn:h_upper_bound_prop:hardness_lowdim_highdim} (and will define the constants $c_0,c_1,c_2$ along the way).

\subsubsection{Verifying~\eqref{eqn:show_N_prop:hardness_lowdim_highdim} for Case 1}
Let $\Ph_N=\frac{1}{N}\sum_{i=1}^N\delta_{X_i}$ be the empirical distribution of data points $X_1,\dots,X_N$, and define
\[t_{d,\epsilon} = F^{-1}_{\chi_d}(1-\epsilon/2),\]
where $F_{\chi_d}$ is the CDF of the $\chi_d$ distribution. We  consider the set
\[A = \cup_{i\in[N]} \mathbb{B}_d(X_i,h\cdot t_{d,\epsilon}),\]
where $\mathbb{B}_d(x,r)$ is the ball with center $x$ and radius $r$. Then, writing $\xi\sim \psi$, by definition of $\Ph_N\ast\psi_h$, we have
\begin{align*}
    (\Ph_N\ast\psi_h)(A)
    &=\frac{1}{N}\sum_{i=1}^N \PP{X_i + h\xi \in A}\\
    &\geq \frac{1}{N}\sum_{i=1}^N \PP{X_i + h\xi \in \mathbb{B}_d(X_i,h\cdot t_{d,\epsilon})}\\
    &= \frac{1}{N}\sum_{i=1}^N \PP{\|h\xi\|_2\leq h\cdot t_{d,\epsilon}}=1-\epsilon/2,
\end{align*}
where the last step holds since $\|\xi\|_2\sim\chi_d$. On the other hand, by definition of $Q\ast\psi_h$, for $Y\sim Q$ and $\xi\sim\psi$ with $Y\independent\xi$, we have
\begin{align*}
    (Q\ast\psi_h)(A)
    &=\PP{Y + h\xi \in A}\\
    &\leq \sum_{i=1}^N \PP{Y + h\xi \in \mathbb{B}_d(X_i,h\cdot t_{d,\epsilon})}\\
    &\leq \sum_{i=1}^N \textnormal{Vol}(\mathbb{B}_d(X_i,h\cdot t_{d,\epsilon})) \cdot C(\textnormal{Vol}(\mathbb{B}_d))^{-1}=CN(ht_{d,\epsilon})^d,
\end{align*}
where the next-to-last step holds since $Q$ has a density satisfying $\sup_y |q(y)|\leq C(\textnormal{Vol}(\mathbb{B}_d))^{-1}$, and convolution cannot increase the maximum value of a density. Therefore, it follows that
\[\dtv^h(\Ph_N,Q)\geq (\Ph_N\ast\psi_h)(A)-(Q\ast\psi_h)(A)\geq 1-\epsilon/2 - CN(ht_{d,\epsilon})^d,\]
for any data points $X_1,\dots,X_N\in\mathbb{B}_d$.
Consequently, if
\[h\leq  \frac{1}{t_{d,\epsilon}}\left(\frac{\epsilon}{2CN}\right)^{1/d}\]
then we have proved that \[\dtv^h(\Ph_N,Q)\geq 1-\epsilon,\]
almost surely, which verifies~\eqref{eqn:show_N_prop:hardness_lowdim_highdim}. 

Now by definition of $N$, we have $N\leq n^2/\delta$.
Since $d\geq 1$, we also have $\left(\frac{\epsilon\delta}{2C}\right)^{1/d}\geq \frac{\epsilon\delta}{2C}$. Moreover, the $\chi_d$ distribution is $1$-subgaussian and has mean $\leq\sqrt{d}$. Hence, it holds that $t_{d,\epsilon}\leq c_\epsilon\sqrt{d}$ where $c_\epsilon$ depends only on $\epsilon$. Together, it suffices to have
\[h\leq \frac{\epsilon\delta}{2Cc_\epsilon\sqrt{d}}\cdot n^{-2/d}.\]
Thus as long as we choose $c_0 \leq \frac{\epsilon\delta}{2Cc_\epsilon}$, we have verified the result if the first case of~\eqref{eqn:h_upper_bound_prop:hardness_lowdim_highdim} holds. Note that this holds for any $d\geq 1$---that is, regardless of the choice of $c_1$.

\subsubsection{Verifying~\eqref{eqn:show_N_prop:hardness_lowdim_highdim} for Case 2}
First let $X_1,\dots,X_N\in\R^d$ be a \emph{fixed} set of data points and let $\Ph_N=\frac{1}{N}\sum_{i=1}^N\delta_{X_i}$ denote the empirical distribution. Define a set
    \[A = \left\{x\in\R^d : \max_{i\in[N]} X_i^\top x \geq 0.5\right\}.\]
    First, we calculate
    \begin{align*}
        (\Ph_N\ast\psi_h)(A)
        &=\frac{1}{N}\sum_{i=1}^N \PP{X_i + h\xi \in A}\textnormal{ where $\xi\sim \psi$}\\
        &= \frac{1}{N}\sum_{i=1}^N \PP{\max_{j\in[N]}X_j^\top(X_i + h\xi)\geq 0.5}\\
        &\geq \frac{1}{N}\sum_{i=1}^N \PP{X_i^\top(X_i + h\xi)\geq 0.5}\\
        &\geq \frac{1}{N}\sum_{i=1}^N \One{\|X_i\|^2_2 \geq 0.75} - \frac{1}{N}\sum_{i=1}^N \PP{h\cdot X_i^\top \xi \leq -0.25}\\
         &\geq \frac{1}{N}\sum_{i=1}^N \One{\|X_i\|^2_2 \geq 0.75} - e^{-(0.25/h)^2/2},
    \end{align*}
    where for the last step, we use the fact that $h\cdot X_i^\top \xi \sim\mathcal{N}(0,h^2\|X_i\|^2_2)$, and $\|X_i\|_2\leq 1$, so $\PP{h\cdot X_i^\top \xi \leq -0.25} \leq \Phi(-0.25/h)\leq e^{-(0.25/h)^2/2}$, where $\Phi$ is the CDF of the standard normal distribution. Choosing $h \leq \frac{1}{\sqrt{32\log(4/\epsilon)}}$, then,
    \[(\Ph_N\ast\psi_h)(A) \geq \frac{1}{N}\sum_{i=1}^N \One{\|X_i\|^2_2 \geq 0.75} - \epsilon/4.\]

    On the other hand,
    \begin{align*}
        (Q\ast\psi_h)(A)
        &= \PP{Y + h\xi \in A}\textnormal{ for $\xi\sim\psi$, $Y\sim Q$ with $Y\independent \xi$}\\
        &= \PP{\max_{j\in[N]}X_j^\top(Y + h\xi)\geq 0.5}\\
        &\leq \PP{\max_{j\in[N]}X_j^\top Y \geq 0.25} + \PP{\max_{j\in[N]}h\cdot X_j^\top \xi \geq 0.25}\\
        &\leq \PP{\max_{j\in[N]}X_j^\top Y \geq 0.25} + N e^{-(0.25/h)^2/2},
    \end{align*}
    since $h\cdot X_j^\top \xi\sim \mathcal{N}(0,h^2\|X_j\|^2_2)$ and $\|X_j\|_2\leq 1$.
    Next, since $Q$ has density $q$ with $\sup_y |q(y)|\leq C(\textnormal{Vol}(\mathbb{B}_d))^{-1}$,
    \begin{align*}
        \PP{\max_{j\in[N]}X_j^\top Y \geq 0.25} 
        &=\int_{y\in\mathbb{B}_d} \One{\max_{j\in[N]}X_j^\top y\geq 0.25}\cdot q(y)\;\mathsf{d}y\\
        &\leq C \cdot \int_{y\in\R^d} \One{\max_{j\in[N]}X_j^\top y\geq 0.25}\cdot \frac{\One{\|y\|_2\leq 1}}{\textnormal{Vol}(\mathbb{B}_d)}\;\mathsf{d}y\\
        &= C \cdot \PP{\max_{j\in[N]}X_j^\top V\geq 0.25}\textnormal{ for $V\sim\textnormal{Unif}(\mathbb{B}_d)$}\\
        &\leq C \cdot \PP{\max_{j\in[N]}X_j^\top U\geq 0.25}\textnormal{ for $U\sim\textnormal{Unif}(\mathbb{S}_{d-1})$}\\
         &\leq C \cdot \PP{\max_{j\in[N]}\left(\frac{X_j}{\|X_j\|_2}\right)^\top U\geq 0.25}\textnormal{ since $\|X_j\|_2\leq 1$ for all $j$}\\
        &\leq CN \cdot 2e^{-(0.25)^2d/2},
    \end{align*}
    where the last step holds since, for any unit vector $v$, the random variable $v^\top U$ is $\frac{1}{d}$-subgaussian \citep[Theorem 3.4.5]{vershynin2018high}.
    Combining everything,
    \[(Q\ast\psi_h)(A) \leq CN \cdot 2e^{-(0.25)^2d/2} + N e^{-(0.25/h)^2/2}\leq\epsilon/2,\]
    where for the last step we assume $d\geq 32\log(8CN/\epsilon)$ and $h\leq \frac{1}{\sqrt{32\log(4N/\epsilon)}}$, so that each term is $\leq\epsilon/4$.
    Therefore,
    \[\dtv^h(\Ph_N,Q) \geq (\Ph_N\ast\psi_h)(A)-(Q\ast\psi_h)(A)\geq  \frac{1}{N}\sum_{i=1}^N \One{\|X_i\|^2_2 \geq 0.75} - \epsilon/4 - \epsilon/2.\]
    Consequently, 
    \[\frac{1}{N}\sum_{i=1}^N\One{\|X_i\|^2_2\geq 0.75} \geq 1-\epsilon/4 \ \Longrightarrow \ \dtv^h(\Ph_N,Q) \geq 1-\epsilon.\]

Finally, by the bounded density assumption on $P$, we have
\begin{multline*}\Pp{P}{\|X\|^2_2 < 0.75} = \int_{x\in\mathbb{B}_d} \One{\|x\|_2< \sqrt{0.75}}\cdot p(x)\;\mathsf{d}x\\\leq 
\int_{x\in\mathbb{B}_d} \One{\|x\|_2< \sqrt{0.75}}\cdot \frac{C}{{\textnormal{Vol}(\mathbb{B}_d)}}\;\mathsf{d}x = C\cdot 0.75^{d/2}.\end{multline*}
Therefore, 
\[\sum_{i=1}^N\One{\|X_i\|^2_2<0.75} \sim\textnormal{Binomial}(N,\rho)\textnormal{ for some $\rho\leq C\cdot 0.75^{d/2}\leq\epsilon/8$},\]
where the last inequality holds if we assume $d\geq \frac{\log(8C/\epsilon)}{0.5\log(4/3)}$.
Consequently
by Hoeffding's inequality, 
\begin{multline*}\PP{\frac{1}{N}\sum_{i=1}^N\One{\|X_i\|^2_2\geq 0.75} \geq 1- \epsilon/4} \\= 1-\PP{\frac{1}{N}\sum_{i=1}^N\One{\|X_i\|^2_2<0.75} > \epsilon/4} \geq 1- e^{-2N(\epsilon/8)^2}.\end{multline*}
 Assuming also that $N\geq \frac{32\log(2/\delta)}{\epsilon^2}$, then,
\[\PP{\frac{1}{N}\sum_{i=1}^N\One{\|X_i\|_2\geq 0.75} \geq 1- \epsilon/4} \geq 1-\delta/2,\]
which completes the proof. To gather our assumptions, we have required that:
\begin{align*}d&\geq \max\left\{32\log(8CN/\epsilon),\frac{\log(8C/\epsilon)}{0.5\log(4/3)}\right\},\\ h&\leq \min\left\{\frac{1}{\sqrt{32\log(4N/\epsilon)}},\frac{1}{\sqrt{32\log(4/\epsilon)}}\right\},\\ N&\geq \frac{32\log(2/\delta)}{\epsilon^2}.\end{align*}
By definition, we have $N\leq n^2/\delta$, and so the conditions on $d$ and $h$ can be relaxed to
\[d\geq \max\left\{32\log(8Cn^2/\epsilon\delta) , \frac{\log(8C/\epsilon)}{0.5\log(4/3)}\right\},\quad h\leq \frac{1}{\sqrt{32\log(4n^2/\epsilon\delta)}}.\]
Assuming that $n \geq \frac{8C}{\epsilon\delta}$ this can be further relaxed to
\[d\geq \max\left\{32\log(n^3),\frac{\log n}{0.5\log(4/3)}\right\},\quad h\leq \frac{1}{\sqrt{32\log(n^3)}}.\]
As long as $c_1 \geq \max\{96,\frac{1}{0.5\log(4/3)}\}$ and $c_0\leq \frac{1}{\sqrt{96}}$, then, these assumptions hold when we assume $d\geq c_1\log n$ and $h\leq c_0/\sqrt{\log n}$. Finally, we have been assuming that 
\[n \geq \frac{8C}{\epsilon\delta}, \quad N\geq \frac{32\log(2/\delta)}{\epsilon^2}.\]
By definition of $N = \left\lceil\frac{n(n-1)}{\delta}\right\rceil$ this can be relaxed to
\[n\geq \max\left\{\frac{8C}{\epsilon\delta} , \sqrt{\frac{64\delta\log(2/\delta)}{\epsilon^2}}\right\},\]
which yields our choice of $c_2$.
\subsubsection{Proof of lemma}\label{app:lemma_proof_curse_dimensionality}

\begin{proof}\textbf{of Lemma~\ref{lem:apply_sample_resample}.}
    First let $X^{(1)},\dots,X^{(N)}\in\R^d$ be a \emph{fixed} set of data points and let $\Ph_N=\frac{1}{N}\sum_{i=1}^N\delta_{X^{(i)}}$ denote the empirical distribution. Then, since $\hat{U}_\alpha$ is a DF-UCB, it must hold that
    \[\Pp{\Ph_N,Q}{\hat{U}_\alpha(\mathcal{D}_{n,m},\zeta) \geq \dtv^h(\Ph_N,Q)} \geq 1-\alpha,\]
    where the notation $\Pp{\Ph_N,Q}{\hdots}$ means that probability is calculated with respect to a dataset $\mathcal{D}_{n,m}$ comprised of $X_1,\dots,X_n\iidsim\Ph_N$ and $Y_1,\dots,Y_m\iidsim Q$. Therefore,
    \[\Pp{\Ph_N,Q}{\hat{U}_\alpha(\mathcal{D}_{n,m},\zeta) \geq 1-\epsilon} \geq 1-\alpha - \One{\dtv^h(\Ph_N,Q)<1-\epsilon}.\]

    Next let $X^{(1)},\dots,X^{(N)}\iidsim P$. Then, after conditioning on these data points (so that their empirical distribution $\Ph_N$ is now fixed),
    \[\Ppst{\Ph_N,Q}{\hat{U}_\alpha(\mathcal{D}_{n,m},\zeta) \geq 1-\epsilon}{X^{(1)},\dots,X^{(N)}} \geq 1-\alpha - \One{\dtv^h(\Ph_N,Q)<1-\epsilon},\]
    and consequently after marginalizing over the draw of $X^{(1)},\ldots, X^{(N)}$,
    \begin{multline*}\EE{\Ppst{\Ph_N,Q}{\hat{U}_\alpha(\mathcal{D}_{n,m},\zeta) \geq 1-\epsilon}{X^{(1)},\dots,X^{(N)}}}\\ \geq 1-\alpha - \PP{\dtv^h(\Ph_N,Q)<1-\epsilon}\geq 1-\alpha-\delta'.\end{multline*}
    Finally, by \citet[Lemma 4.16]{angelopoulos2024theoretical}, the total variation distance between sampling $X_1,\dots,X_n\iidsim P$, versus sampling $\Ph_N$ and then sampling $X_1,\dots,X_n\iidsim \Ph_N$, is at most $\frac{n(n-1)}{2N}$. Therefore,
    \begin{multline*}\Pp{P,Q}{\hat{U}_\alpha(\mathcal{D}_{n,m},\zeta) \geq 1-\epsilon} 
   \\ \geq \EE{\Ppst{\Ph_N,Q}{\hat{U}_\alpha(\mathcal{D}_{n,m},\zeta) \geq 1-\epsilon}{X^{(1)},\dots,X^{(N)}}} - \frac{n(n-1)}{2N},\end{multline*}
    which completes the proof.
\end{proof}

\subsection{Proof of Theorem~\ref{thm:blurred_tv_depends_on_eff_dimension}: distributions supported on a low-dimensional subspace}

We start by considering the case that $P$ and $Q$ are supported on a $k$-dimensional linear subspace whose span is given by an orthonormal matrix $A\sim \R^{d\times k}$.
Fix $h>0$ and an orthonormal matrix $A\in\R^{d\times k}$. Let $B\in\R^{d\times (d-k)}$ be an orthonormal matrix spanning the orthogonal complement of the span of $A$.
Let $\xi\sim\psi_d=\mathcal{N}(0,I_d)$, and note that $A^\top\xi\sim\psi_k=\mathcal{N}(0,I_k)$. Then, for $X\sim P$ drawn independently of $\xi$,
\[A^\top (X+h\xi) = A^\top X + h \cdot A^\top \xi \sim (A\circ P)\ast (\psi_k)_h.\]
And, 
\[B^\top (X+h\xi) = B^\top X + h\cdot B^\top \xi = h\cdot B^\top \xi\sim (\psi_{d-k})_h,\]
by our assumption on the support of $P$. Moreover, by properties of the Gaussian distribution, $A^\top\xi\independent B^\top \xi$, and thus $A^\top (X+h\xi)\independent B^\top (X+h\xi)$.
Therefore, 
\[\big(A^\top (X+h\xi),B^\top (X+h\xi)\big)\sim \big((A\circ P)\ast (\psi_k)_h \big) \times (\psi_{d-k})_h.\]
A similar calculation holds with $Y$ in place of $X$. Therefore,
\begin{align*}
    \dtv^{h,\psi_d}(P,Q)
    &=\dtv(X+h\xi, Y+h\xi)\\
    &=\dtv\left( (A \; B)^\top (X+h\xi) , (A \; B)^\top (Y+h\xi)\right)\textnormal{ since $(A\; B)\in\R^{d\times d}$ is invertible}\\
    &=\dtv\left( \big((A\circ P)\ast (\psi_k)_h \big) \times (\psi_{d-k})_h, \big((A\circ Q)\ast (\psi_k)_h \big) \times (\psi_{d-k})_h\right)\\
    &=\dtv((A\circ P)\ast (\psi_k)_h,(A\circ Q)\ast (\psi_k)_h)\\
    &=\dtv^{h,\psi_k}(A\circ P, A\circ Q).
\end{align*}

Now, moving on to the general case, consider $P,Q\in \mathcal{P}_d$ and an orthonormal matrix $A\in \R^{d\times k}$. Now, let us write $\tilde{P}$ and $\tilde{Q}$ to denote the distributions of $AA^\top X$ and $AA^\top Y$ respectively, and note that these distributions are supported on the linear space spanned by $A$. By the above argument, we note that 
\[
\dtv^{h,\psi_d}(\tilde{P},\tilde{Q})=\dtv^{h,\psi_k}(A\circ \tilde{P}, A\circ \tilde{Q}).
\]
Observe that since $A$ is orthonormal, $A^\top AA^\top X= A^\top X$ and $A^\top AA^\top Y= A^\top Y$. Therefore, $A\circ \tilde{P}=A\circ P$ and $A\circ \tilde{Q}=A\circ Q$. Further, by the triangle inequality, we have that 
\[
\bigl|\dtv^{h,\psi_d}(P,Q)- \dtv^{h,\psi_d}(\tilde{P},\tilde{Q})\bigr|\le \dtv^{h,\psi_d}(P,\tilde{P})+\dtv^{h,\psi_d}(Q,\tilde{Q}).
\]
Together, we have that
\[
\bigl|\dtv^{h,\psi_d}(P,Q)- \dtv^{h,\psi_k}(A\circ P, A\circ Q)\bigr|\le \dtv^{h,\psi_d}(P,\tilde{P})+\dtv^{h,\psi_d}(Q,\tilde{Q}).
\]
Finally, note that 
\begin{align*}
    \dtv^{h,\psi_d}(P,\tilde{P})&=\dtv(X+h\xi, AA^\top X+h\xi)\\
    &\le \EE{\dtv\left(X+h\xi, AA^\top X+h\xi\ \middle| \ X\right)}\\ 
    &=\EE{\dtv\left(\frac{X-AA^\top X}{h}+\xi,\xi\ \middle| \ X\right)} =\EE{\omega_{\psi_d}\left(\frac{X-AA^\top X}{h}\right)},
\end{align*}
where recall that $\omega_{\psi_d}(v):=\dtv(\delta_v\ast\psi_d,\psi_d)$ is the shift modulus of kernel $\psi_d$. Since $\psi_d$ is the density of $\mathcal{N}(0,I_d)$, we have that 
\[
\omega_{\psi_d}\left(\frac{X-AA^\top X}{h}\right)\le \frac{\|X-AA^\top X\|_2}{\sqrt{2\pi}h}.
\]
Hence,  it follows that 
\[
\dtv^{h,\psi_d}(P,\tilde{P})\le \EE{\frac{\|X-AA^\top X\|_2}{\sqrt{2\pi}h}},
\]
and similarly,
\[
\dtv^{h,\psi_d}(Q,\tilde{Q})\le \EE{\frac{\|Y-AA^\top Y\|_2}{\sqrt{2\pi}h}}.
\]
This completes the proof.
\hfill$\square$

\section{Additional results and extensions}\label{app:additional_results}

\subsection{Bounding difference between TV and blurred TV}
\label{app:blurred_tv_approximation}

Proposition~\ref{prop:blurred_tv&tv_properties}~(ii) shows that, for any two distributions $P,Q\in\mathcal{P}_d$, the blurred TV distance $\dtv^h(P,Q)$ converges to the original TV distance $\dtv(P,Q)$ as $h\to0$. Here, we provide a quantitative bound on this approximation error under a mild smoothness condition on the underlying distributions.

In particular, let $\mathcal{P}_{\textnormal{smooth},d}\subseteq\mathcal{P}_d$ denote the class of distributions $P$ that admit a continuously differentiable density $p$ with respect to the Lebesgue measure on $\mathbb{R}^d$, satisfying
\[
\int_{\mathbb R^d}\|\nabla p(x)\|_2\,\mathrm{d}x<\infty,
\]
where $\nabla p$ denotes the gradient of $p$.

\begin{theorem}\label{thm:blurred_tv_approximation}
Let $P\in\mathcal P_{\textnormal{smooth},d}$, with density $p$, and suppose that $\E_{\xi\sim\psi}[\|\xi\|_2]<\infty$. Then, for every $h>0$,
\begin{equation*}
\dtv(P,P\ast\psi_h)
\leq
\frac{h}{2}\E_{\xi\sim\psi}[\|\xi\|_2]
\int_{\mathbb R^d}\|\nabla p(x)\|_2\,\mathrm{d}x.
\end{equation*}
Consequently, if $Q\in\mathcal P_{\textnormal{smooth},d}$ with density $q$, then
\begin{equation}\label{eq:tv_blurred_tv_gap}
0\leq \dtv(P,Q)-\dtv^h(P,Q)
\leq \frac{h}{2}\E_{\xi\sim\psi}[\|\xi\|_2]
\left\{
\int_{\mathbb R^d}\|\nabla p(x)\|_2\,\mathrm{d}x
+
\int_{\mathbb R^d}\|\nabla q(x)\|_2\,\mathrm{d}x
\right\}.
\end{equation}
\end{theorem}

\begin{proof}
Let $p$ denote the density of $P$ with respect to the Lebesgue measure on $\mathbb{R}^d$. Then $P\ast\psi_h$ admits the density
\[
x\mapsto \E_{\xi\sim\psi}[p(x-h\xi)].
\]
Thus, by Jensen's inequality,
\[
\dtv(P,P\ast\psi_h)
=
\frac{1}{2}\int_{\mathbb R^d}
\left|p(x)-\E_{\xi\sim\psi}[p(x-h\xi)]\right|
\,\mathrm{d}x
\leq
\frac{1}{2}\E_{\xi\sim\psi}
\left[
\int_{\mathbb R^d}
|p(x)-p(x-h\xi)|
\,\mathrm{d}x
\right].
\]

Now, for any $u\in\mathbb R^d$, define $g(t):=p(x-tu)$ for $t\in[0,1]$. By the chain rule, $g'(t)=-\left\langle \nabla p(x-tu),u\right\rangle$.
Therefore, by the fundamental theorem of calculus and by the Cauchy--Schwarz inequality,
\[
p(x-u)-p(x)
=-\int_0^1
\left\langle \nabla p(x-tu),u\right\rangle\,\mathrm{d}t\leq
\|u\|_2
\int_0^1
\|\nabla p(x-tu)\|_2\,\mathrm{d}t.
\]
Integrating over $x$, and applying Tonelli's theorem therefore yields
\begin{align*}
\int_{\mathbb R^d}|p(x-u)-p(x)|\,\mathrm{d}x
\leq
\|u\|_2
\int_0^1
\int_{\mathbb R^d}
\|\nabla p(x-tu)\|_2
\,\mathrm{d}x\,\mathrm{d}t=
\|u\|_2
\int_{\mathbb R^d}
\|\nabla p(z)\|_2\,\mathrm{d}z.
\end{align*}
Applying this bound with $u=h\xi$ and taking expectation over $\xi$ gives
\[
\dtv(P,P\ast\psi_h)
\leq
\frac{h}{2}\E_{\xi\sim\psi}[\|\xi\|_2]
\int_{\mathbb R^d}\|\nabla p(x)\|_2\,\mathrm{d}x,
\]
which proves the first claim.

For the second claim, Proposition~\ref{prop:blurred_tv&tv_properties}~(i) gives $\dtv^h(P,Q)\leq \dtv(P,Q)$. Moreover, by the triangle inequality,
\[
\dtv(P,Q)
\leq
\dtv(P,P\ast\psi_h)
+
\dtv^h(P,Q)
+
\dtv(Q\ast\psi_h,Q).
\]
Therefore, applying the first claim to $P$ and $Q$ proves~\eqref{eq:tv_blurred_tv_gap}.
\end{proof}

Theorem~\ref{thm:blurred_tv_approximation} quantifies the approximation error incurred when replacing the original TV distance with blurred TV. In particular, over any class of distributions for which the integrals of the Euclidean norms of the density gradients are uniformly bounded, the approximation error is $O(h)$.

\subsection{Convergence of the empirical distribution in blurred TV}
In Theorem~\ref{thm:blurred_TV_convergence_asymptotic}, we showed that $\dtv^h(\Ph_n,P)$ vanishes as $n\to \infty$. In this section, we give a more quantitative version of this result.

\begin{theorem}\label{thm:blurred_TV_convergence_with_rate}
Let $P\in\mathcal{P}_d$ and $\psi\in\mathcal{K}_d$. Assume $\psi$ is a bounded density, with $\sup_{x\in\R^d}\psi(x)\leq B$. Then for any bandwidth $h>0$,
\[\EE{\dtv^h(\Ph_n,P)}\leq \inf_{A\subseteq\R^d}\left\{\frac{B\cdot \textnormal{Vol}(A)}{h^d}\cdot n^{-1/2} + (P\ast\psi_h)(A^c)\right\},\]
where the infimum is taken over all measurable subsets $A\subseteq\R^d$.
\end{theorem}
In particular, if $P\ast\psi_h$ satisfies a moment bound, $\EE{\|X+h\xi\|_2^q}^{1/q}\leq M_q$ for some $q>0$, then choosing $A=\mathbb{B}_d(0,r)$, we have $(P\ast\psi_h)(A^c) = \PP{\|X+h\xi\|_2>r}\leq \frac{M_q^q}{r^q}$ by Markov's inequality. Therefore
\[\EE{\dtv^h(\Ph_n,P)}\leq \frac{B\cdot \textnormal{Vol}(\mathbb{B}_d) \cdot r^d}{h^d}\cdot n^{-1/2} + \frac{M_q^q}{r^q},\]
where $\mathbb{B}_d$ is the unit ball. Choosing
$r = \left(\frac{M_q^qh^dn^{1/2}}{B\cdot\textnormal{Vol}(\mathbb{B}_d)}\right)^{1/(d+q)}$,
then,
\[\EE{\dtv^h(\Ph_n,P)}\leq n^{-\frac{q}{2(d+q)}} \cdot 2M_q^{\frac{dq}{d+q}}\left(\frac{B\cdot \textnormal{Vol}(\mathbb{B}_d)}{h^d}\right)^{\frac{q}{d+q}}.\]
\begin{proof}
    We will use the following facts: the density of $P\ast\psi_h$ is given by $z\mapsto \Ep{X\sim P}{\psi_h(z-X)}$, and, for any distributions $P,Q$ with densities $f,g$ we can calculate $\dtv(P,Q) = \int (f(z)-g(z))_+\;\mathsf{d}z$.

    We then have
    \begin{multline*}
\Ep{X_1,\ldots,X_n\iidsim P}{\dtv^h(\Ph_n,P)}
=\Ep{X_1,\ldots,X_n\iidsim P}{\dtv(\Ph_n\ast\psi_h,P\ast\psi_h)}\\
=\Ep{X_1,\ldots,X_n\iidsim P}{\int
\bigl(\Ep{X\sim P}{\psi_h(z-X)}-\Ep{X\sim \Ph_n}{\psi_h(z-X)}\bigr)_+\,\mathsf{d}z}.
\end{multline*}
Now fix a subset $A\subseteq\R^d$. For any $z\in \R^d$, we have
\[\bigl(\Ep{X\sim P}{\psi_h(z-X)}-\Ep{X\sim \Ph_n}{\psi_h(z-X)}\bigr)_+ \leq \Ep{X\sim P}{\psi_h(z-X)},\]
and therefore
\begin{multline*}\int_{A^c}
\bigl(\Ep{X\sim P}{\psi_h(z-X)}-\Ep{X\sim \Ph_n}{\psi_h(z-X)}\bigr)_+\,\mathsf{d}z\\\leq \int_{A^c}\Ep{X\sim P}{\psi_h(z-X)}\;\mathsf{d}z = (P\ast\psi_h)(A^c),\end{multline*}
almost surely.
This yields
\begin{align*}
    &\Ep{X_1,\ldots,X_n\iidsim P}{\dtv^h(\Ph_n,P)}\\
    &\leq (P\ast\psi_h)(A^c) + \Ep{X_1,\ldots,X_n\iidsim P}{\int_A
\bigl(\Ep{X\sim P}{\psi_h(z-X)}-\Ep{X\sim \Ph_n}{\psi_h(z-X)}\bigr)_+\,\mathsf{d}z}\\
&= (P\ast\psi_h)(A^c) + \int_A\Ep{X_1,\ldots,X_n\iidsim P}{
\bigl(\Ep{X\sim P}{\psi_h(z-X)}-\Ep{X\sim \Ph_n}{\psi_h(z-X)}\bigr)_+}\,\mathsf{d}z,
\end{align*}
where the last step holds by the Fubini--Tonelli theorem.
Next, for any $z\in\R^d$, we can calculate
\begin{align*}
&\hspace{-1.5em}\Ep{X_1,\ldots,X_n\iidsim P}{
\bigl(\Ep{X\sim P}{\psi_h(z-X)}-\Ep{X\sim \Ph_n}{\psi_h(z-X)}\bigr)_+}\\
&\leq\Ep{X_1,\ldots,X_n\iidsim P}{
\bigl(\Ep{X\sim\Ph_n}{\psi_h(z-X)}-\Ep{X\sim P}{\psi_h(z-X)}\bigr)^2}^{1/2}\\
&=\mathrm{Var}_{X_1,\ldots,X_n\iidsim P}
\left(\Ep{X\sim\Ph_n}{\psi_h(z-X)}\right)^{1/2}\\
&=\left[\frac1n\mathrm{Var}_{X_1\sim P}\bigl(\psi_h(z-X_1)\bigr)\right]^{1/2} \leq \frac{B}{h^d}\cdot \frac{1}{n^{1/2}},
\end{align*}
where the last step holds since $\sup_{z\in\R^d}\psi_h(z) = \sup_{z\in\R^d}h^{-d}\psi(z/h)\leq B/h^d$. Therefore,
\begin{multline*}\int_A\Ep{X_1,\ldots,X_n\iidsim P}{
\bigl(\Ep{X\sim P}{\psi_h(z-X)}-\Ep{X\sim \Ph_n}{\psi_h(z-X)}\bigr)_+}\,\mathsf{d}z\\
\leq \int_A \frac{B}{h^d}\cdot \frac{1}{n^{1/2}}\;\mathsf{d}z = \frac{B}{h^d}\cdot \frac{1}{n^{1/2}}\cdot \textnormal{Vol}(A),\end{multline*}
which completes the proof.
\end{proof}

\subsection{Distribution-free \dfucb that is bandwidth-adaptive and uniform over bandwidths}
\label{app:confidence_bounds_all_properties}

In this section, we combine the strategies developed earlier to construct a \dfucb for blurred TV that is simultaneously
(a) computationally efficient,
(b) uniformly valid over all bandwidths $h>0$, and
(c) adaptive to the choice of $h$. (An analogous construction is also possible to derive a \dflcb, but we do not state the details here.)

We begin by defining a Monte Carlo estimator of the monotonized blurred TV:
\[
\widehat{\dtv^{h,\mathrm{up}}}(\Ph_n,\Qh_m;B)
:=
\sup_{h_1 \in [h,\infty)}\widehat{\dtv^{h_1}}(\Ph_n,\Qh_m;B),
\]
where $\widehat{\dtv^{h_1}}(\Ph_n,\Qh_m;B)$ is the  Monte Carlo estimator defined in Section~\ref{sec:monte_carlo_approximation}.
Define also
\[\hat\Sigma_{n,m}^{h,\mathrm{up}} = \sup_{h_1\in[h,\infty)}\hat\Sigma_{n,m}^{h_1}\]
where $\hat\Sigma_{n,m}^h$ is defined as in Section~\ref{sec:confidence_bounds_bandwidth_adaptive}.

\begin{theorem}\label{thm:UCB_and_LCB_monotonized_TV}
Fix any $d\geq 1$, $n,m\geq 1$, $\psi\in\mathcal{K}_d$, $h>0$, and $\alpha\in(0,1)$. Define 
\begin{multline*}
\hat{U}^{h,\mathrm{up}}_\alpha(\mathcal{D}_{n,m},\hat{\Sigma}^{h,\mathrm{up}}_{n,m},B) = \widehat{\dtv^{h,\mathrm{up}}}(\Ph_n,\Qh_m;B)
+ \bigl(\hat\Sigma_{n,m}^{h,\mathrm{up}}\bigr)^{1/2}\,r_{1,\alpha/2(n\wedge m)}\\
+ \frac{1}{n\wedge m}
\left(\frac{r_{1,\alpha/2(n\wedge m)}^2}{3}
+ 2\sqrt{5}\,r_{1,\alpha/2(n\wedge m)} r_{2,\alpha/2(n\wedge m)}\right)+ \epsilon_{B,\alpha/2(n\wedge m)} + \frac{1}{n\wedge m},
\end{multline*}
Then $\hat{U}^{h,\mathrm{up}}_\alpha(\mathcal{D}_{n,m},\hat{\Sigma}^{h,\mathrm{up}}_{n,m},B)$ is a \dfucb for $\dtv^h(\cdot,\cdot)$, uniformly over all $h\in(0,\infty)$, at confidence level $1-\alpha$.
\end{theorem}

\begin{proof}
We build the confidence bound in stages.
We start from the bandwidth-adaptive \dfucb in
Theorem~\ref{thm:ucb_bandwidth_adaptive}, and then extend it to hold with a Monte Carlo approximation as in Theorem~\ref{thm:monte_carlo_ucb}. We then convert to a uniform bound (following the approach of Appendix~\ref{app:proof_of_uniform_confidence_bounds} for the proof of Theorem~\ref{thm:ucb_uniform} as a final step.

\paragraph{Constructing a Monte Carlo, bandwidth-adaptive \dfucb on $\dtv^h(P,Q)$.}
Fix any $h>0$. First define
\[
\hat{U}_{\alpha/2}(\mathcal{D}_{n,m},\hat\Sigma_{n,m}^{h})=\dtv^h(\Ph_n,\Qh_m)
+ \bigl(\hat\Sigma_{n,m}^{h}\bigr)^{1/2} r_{1,\alpha/2}
+ \frac{1}{n\wedge m}
\left(\frac{r_{1,\alpha/2}^2}{3}
+ 2\sqrt{5}\,r_{1,\alpha/2} r_{2,\alpha/2}\right)
\]
exactly as in Theorem~\ref{thm:ucb_bandwidth_adaptive} (but with $\alpha/2$ in place of $\alpha$). 
The result of that theorem tells us that
\[\PP{\dtv^h(P,Q)\leq \hat{U}_{\alpha/2}(\mathcal{D}_{n,m},\hat\Sigma_{n,m}^{h})} \geq 1-\alpha/2.\]
Next, as in the proof of Theorem~\ref{thm:monte_carlo_ucb},
\[
\PP{\dtv^{h}(\Ph_n,\Qh_m)
\le \widehat{\dtv^{h}}(\Ph_n,\Qh_m;B)
+ \epsilon_{B,\alpha/2}}\geq 1-\alpha/2.
\]
Therefore, defining
\begin{multline*}\hat{U}^h_\alpha(\mathcal{D}_{n,m},\hat\Sigma^h_{n,m},B) =  \widehat{\dtv^{h}}(\Ph_n,\Qh_m;B)\\
+ \bigl(\hat\Sigma_{n,m}^{h}\bigr)^{1/2} r_{1,\alpha/2}
+ \frac{1}{n\wedge m}
\left(\frac{r_{1,\alpha/2}^2}{3}
+ 2\sqrt{5}\,r_{1,\alpha/2} r_{2,\alpha/2}\right)+ \epsilon_{B,\alpha/2},\end{multline*}
we obtain
\begin{equation}\label{eqn:build_up_combined_UCB}\PP{\dtv^h(P,Q)\leq \hat{U}^h_\alpha(\mathcal{D}_{n,m},\hat\Sigma_{n,m}^{h},B)} \geq 1-\alpha,\end{equation}
by a union bound. Hence, $\hat{U}^h_\alpha(\mathcal{D}_{n,m},\hat\Sigma_{n,m}^{h},B)$ is a \dfucb (for any fixed bandwidth $h>0$), that is bandwidth-adaptive and is constructed via a Monte Carlo approximation.

\paragraph{Converting to a uniform bound.}
Following an identical argument as in the proof of Theorem~\ref{thm:ucb_uniform}, we can leverage the fact that~\eqref{eqn:build_up_combined_UCB} holds for any fixed $h>0$ and any $\alpha\in(0,1)$ to obtain
\[\PP{\dtv^h(P,Q)\leq \sup_{h_1\in[h,\infty)}\hat{U}^{h_1}_{\alpha/(n\wedge m)}(\mathcal{D}_{n,m},\hat\Sigma_{n,m}^{h_1},B)+\frac{1}{n\wedge m}\ \forall \ h\in(0,\infty)}\geq 1-\alpha.\]
Noting that $\hat\Sigma^{h,\mathrm{up}}_{n,m} \geq \hat\Sigma_{n,m}^{h_1}$ for all $h_1\in[h,\infty)$, by construction, this implies
\[\sup_{h_1\in[h,\infty)}\hat{U}^{h_1}_{\alpha/(n\wedge m)}(\mathcal{D}_{n,m},\hat\Sigma_{n,m}^{h_1},B)
\leq \sup_{h_1\in[h,\infty)}\hat{U}^{h_1}_{\alpha/(n\wedge m)}(\mathcal{D}_{n,m},\hat\Sigma_{n,m}^{h,\mathrm{up}},B).\]
Combining all  these calculations completes the proof.
\end{proof}

\end{document}

%% file: highdim_illustration_tikz.tex
\scalebox{0.9}{
\begin{tikzpicture}

\tikzmath{\y0 = 0; \y1 = -2.5; \ymid = -1.25;}
    \tikzmath{\xstart = 0; \xend=12;}
    \tikzmath{\xlo0 = 2; \xlo1 = 7; \xhi=9;}
    \tikzmath{\ytick = 0.15;}
    \tikzmath{\xbrace = 0.1; \ybrace=0.1;}
    \tikzmath{\xlab1 = 2.25; \xlab2 = 10.5;}
    \tikzmath{\xmida1 = 1; \xmidb1 = 3.5; \xmid2 = 10.5;}
    \tikzmath{\xgap = 0.6; \ygap = 0.5;}
    
    \draw[thick, ->] (\xstart, \y0) -- (\xend, \y0);
    \draw[thick] (\xstart, \y0-\ytick) -- (\xstart, \y0+\ytick) node[anchor=south,yshift=3] {$h=0$};
    \draw[thick] (\xlo0, \y0-\ytick) -- (\xlo0, \y0+\ytick) node[anchor=south] {$h\asymp \frac{1}{n^{2/d}}$};
    \draw[thick] (\xhi, \y0-\ytick) -- (\xhi, \y0+\ytick) node[anchor=south,yshift=3] {$h\asymp 1$};

    \draw[thick, ->] (\xstart, \y1) -- (\xend, \y1);
    \draw[thick] (\xstart, \y1+\ytick) -- (\xstart, \y1-\ytick) node[anchor=north,yshift=-3] {$h=0$};
    \draw[thick] (\xlo1, \y1+\ytick) -- (\xlo1, \y1-\ytick) node[anchor=north] {$h\asymp \frac{1}{\sqrt{\log n}}$};
    \draw[thick] (\xhi, \y1+\ytick) -- (\xhi, \y1-\ytick) node[anchor=north,yshift=-3] {$h\asymp 1$};

    \node[align = center] at (\xlab1,\ymid) {\begin{tabular}{c}Impossible to construct\\a meaningful UCB\end{tabular}};
    \node[align = center] at (\xlab2,\ymid) {\begin{tabular}{c}Blurred TV\\is trivial\end{tabular}};

    \draw[thick, ->] (\xlab1-\xgap,\ymid+\ygap) -- (\xmida1,\y0-\ygap);
    \draw[thick, ->] (\xlab1+\xgap,\ymid-\ygap) -- (\xmidb1,\y1+\ygap);
    \draw[thick, ->] (\xlab2,\ymid+\ygap) -- (\xmid2,\y0-\ygap);
    \draw[thick, ->] (\xlab2,\ymid-\ygap) -- (\xmid2,\y1+\ygap);

    \draw [decorate,decoration={brace,amplitude=10pt}, thick, rotate=180](-\xlo0+\xbrace,-\y0+\ybrace) -- (-\xstart-\xbrace,-\y0+\ybrace);
    \draw [decorate,decoration={brace,amplitude=10pt}, thick, rotate=180](-\xend+\xbrace,-\y0+\ybrace) -- (-\xhi-\xbrace,-\y0+\ybrace);
    \draw [decorate,decoration={brace,amplitude=10pt}, thick](\xstart+\xbrace,\y1+\ybrace) -- (\xlo1-\xbrace,\y1+\ybrace);
    \draw [decorate,decoration={brace,amplitude=10pt}, thick](\xhi+\xbrace,\y1+\ybrace) -- (\xend-\xbrace,\y1+\ybrace);

    \node[align = center] at (\xstart-2.5, \y0) {\begin{tabular}{c}Low-dim.\\ ($d=\mathcal{O}(1)$)\end{tabular}};

    \node[align = center] at (\xstart-2.5, \y1) {\begin{tabular}{c}High-dim.\\ ($d\gg \log n$)\end{tabular}};

\end{tikzpicture}}